\documentclass[11pt,a4paper]{article}

\usepackage[T1]{fontenc}
\usepackage[utf8]{inputenc}
\usepackage{lmodern}
\usepackage{geometry}
\usepackage{microtype}
\usepackage{bm}
\usepackage{indentfirst}
\usepackage{graphicx}
\usepackage{multirow}
\usepackage{makecell}
\usepackage{array}
\usepackage{booktabs}
\usepackage{algorithm}
\usepackage{algorithmic}
\usepackage{csquotes}
\usepackage{subcaption}
\usepackage{amsmath,amssymb,amsfonts,amsthm,mathtools}
\usepackage[numbers,sort&compress]{natbib}
\usepackage{xcolor}
\usepackage{hyperref}
\hypersetup{colorlinks=true,citecolor=blue,linkcolor=blue,urlcolor=blue}
\graphicspath{{figures/}}

\theoremstyle{plain}
\newtheorem{theorem}{Theorem}[section]
\newtheorem{lemma}[theorem]{Lemma}
\newtheorem{corollary}[theorem]{Corollary}
\newtheorem{proposition}[theorem]{Proposition}

\newtheorem{assumption}[theorem]{Assumption}
\theoremstyle{definition}

\theoremstyle{remark}

\begin{document}

\title{Meta-LinEXP3: Online-within-Online Learning for Adversarial Linear Contextual Bandits}
\author{Hao Li, Jie Xu, Zheng Xie\thanks{Corresponding author: \texttt{xiezheng81@nudt.edu.cn}.} \\[0.5em]
\small College of Science, National University of Defense Technology, Changsha 410073, China}
\date{}
\maketitle

\begin{abstract}
Meta-learning has emerged as an effective paradigm for transferring knowledge across sequential bandit tasks. While substantial progress has been made for stochastic bandits and non-contextual adversarial bandits, meta-learning for adversarial linear contextual bandits (ALCBs) with random action sets remains largely unexplored. To address this problem, we propose Meta-LinEXP3, an online-within-online algorithm that constructs a predictable task-level prior from completed tasks to guide the inner LinEXP3 learner. For known context distributions, we develop a policy-centered estimator that achieves an intrinsic-dimension $\mathcal{O}(\sqrt{n})$ per-task regret bound. For unknown distributions, we introduce a past-only regularized moment estimator with an $\mathcal{O}(n^{2/3})$ leading regret term and explicit finite-sample error. We further establish a direct connection between prior accuracy and transfer regret, showing that increasingly accurate priors yield sublinear transfer-dependent regret across tasks. Experiments demonstrate the effectiveness of Meta-LinEXP3, including its application to structured hyperspectral tensor sampling.
\end{abstract}

\noindent\textbf{Keywords:} Meta-learning, adversarial linear contextual bandits, online-within-online learning, LinEXP3, tensor sampling

\section{Introduction}
Contextual bandits model sequential decision problems in which the learner observes context-dependent actions and receives feedback only for the selected action \cite{CMAB1,CMAB2,CMAB3}. They are widely used for applications such as treatment selection \cite{app1}, personalized recommendation \cite{app2}, and online advertising \cite{CMAB2}. In the stochastic setting, losses are generated from a stationary model. In the adversarial setting, the loss sequence may vary arbitrarily over time.

We focus on adversarial linear contextual bandits (ALCBs). Each action is represented by a context vector, and its loss is the inner product between that context and a loss vector chosen by the environment. Apart from the geometric and non-anticipation conditions stated below, we impose no stationary model on the loss sequence.

A central limitation of the single-task formulation is that many applications naturally generate a sequence of related decision problems. Recommendation systems serve different user cohorts, pricing policies are deployed over successive sales cycles, and clinical studies may involve multiple patient groups. Restarting an ALCB learner on every task discards information that could be useful when task horizons are short. Existing ALCB methods \cite{ALCB,Liu} analyze one task at a time and do not address this form of cross-task transfer.

Meta-learning \cite{LTL} provides a natural framework for exploiting repeated structure across tasks. In online-within-online meta-learning \cite{OWO,OWO2,O_22}, an outer learner updates shared information from completed tasks, while an inner learner makes the round-by-round decisions within each task.

\textbf{Contributions.} Building on \cite{ALCB,K,O}, we study an online-within-online ALCB problem with $m$ tasks of $n$ rounds under a shared context distribution. The loss vector at a round may depend on the entire preceding history, including earlier tasks. The only freshness requirement is on the current context set. Our main contributions are:
\begin{enumerate}
    \item We introduce Meta-LinEXP3, which couples an outer task-level transfer mechanism with an inner adversarial LinEXP3 learner.
    \item We propose positive-cosine retrieval weighting (PCRW) and use a uniform baseline for comparison. Both form a predictable prior from completed task summaries and keep it fixed throughout the next task.
    \item For known context distributions, PC-KDE gives an intrinsic-dimension $\mathcal O(\sqrt n)$ per-task bound. For unknown distributions, PRME gives an $\mathcal O(n^{2/3})$ leading term with an explicit moment estimation error. For structured sequences of oblivious tasks, the stated margin and prior-accuracy conditions imply sublinear transfer-dependent terms in the number of tasks.
    \item We evaluate the method on bounded synthetic tasks, a controlled PC-KDE/PRME/LPE comparison, MovieLens recommendation, and structured KSC tensor sampling, with experiments designed to separate transfer effects from estimator behavior.
\end{enumerate}
These guarantees rely on the fixed-within-task prior and on the specific estimator properties used in the analysis. A different inner estimator requires a separate verification of selection correction, score range, variance, and bias.

Section 2 reviews related work, Section 3 formalizes the problem, Sections 4--5 present the algorithm and regret analysis, and Section 6 reports the experiments. Section 7 concludes, and complete proofs are given in Appendix~\ref{app:proofs}.

\section{Related Work}
\textbf{Adversarial Linear Contextual Bandits.} Early computationally efficient ALCB methods, including ROBUST LinEXP3 and REAL LinEXP3, assume a known context distribution \cite{ALCB}. Subsequent work obtained refined data-dependent guarantees \cite{FS}. When the context distribution is unknown, Liu et al. \cite{Liu} achieve $\widetilde{\mathcal O}(d^2\sqrt n)$ regret without a simulator using lifted log-determinant FTRL, while van Erven et al. \cite{vanErven2025} obtain $\widetilde{\mathcal O}(\min\{d^2\sqrt n,\sqrt{d^3n\log k}\})$ through a reduction to misspecification-robust adversarial linear bandits. These single-task rates are sharper than the general $n^{2/3}$ rate obtained here for PRME. Our objective is different: PRME is based on a policy-independent streaming second moment, which makes cross-task context reuse and the dependence on the meta-prior explicit in the regret bound. Neu et al. \cite{NeuKernel2024} further extend adversarial contextual-bandit methods to reproducing-kernel Hilbert spaces.

\textbf{Meta-Learning for Bandit Problems.}
Most work on meta-learning for bandits assumes stochastic rewards. Representative examples include stochastic linear-bandit meta-learning \cite{O_11}, meta-learned exploration--exploitation strategies \cite{O_21}, Meta-Thompson Sampling \cite{O_19}, prior-update methods \cite{O_25}, transfer through shared affine subspaces \cite{Bilaj2024}, classification-based meta-learning \cite{Mutti2025}, and representation-based linear-bandit methods \cite{O_28,O_6,O_24}. Their guarantees do not extend directly to adversarial within-task loss sequences. In the adversarial setting, Khodak et al. \cite{K} develop online-within-online meta-algorithms for multi-armed bandits and bandit linear optimization, while Osadchiy et al. \cite{O} exploit non-uniform optimal-action distributions across tasks. Neither analysis covers the random-action-set ALCB model considered here.

Existing adversarial online-within-online guarantees \cite{K,O} focus on fixed-arm multi-armed bandits or bandit linear optimization and do not cover the repeated random-action-set ALCB model formalized in Section~\ref{sec2_1}. This setting introduces two additional difficulties: both the comparator action and the sampling design depend on the current random context set, and loss vector estimation must account for the resulting selection bias. Margin assumptions have also been used in contextual-bandit theory to weaken uniform gap conditions. For example, Kato and Ito \cite{KatoIto2025} obtain intermediate best-of-both-worlds rates under a margin condition. We use a margin condition for a more specific purpose: it controls the Laplace transform of the random action gap and converts prior-prediction accuracy into transfer complexity. The cross-task analysis allows arbitrary loss sequences fixed before each task. Without predictable task structure, however, total regret can still grow linearly in the number of tasks.

\begin{figure}[t]
    \centering
    \includegraphics[width=0.95\linewidth]{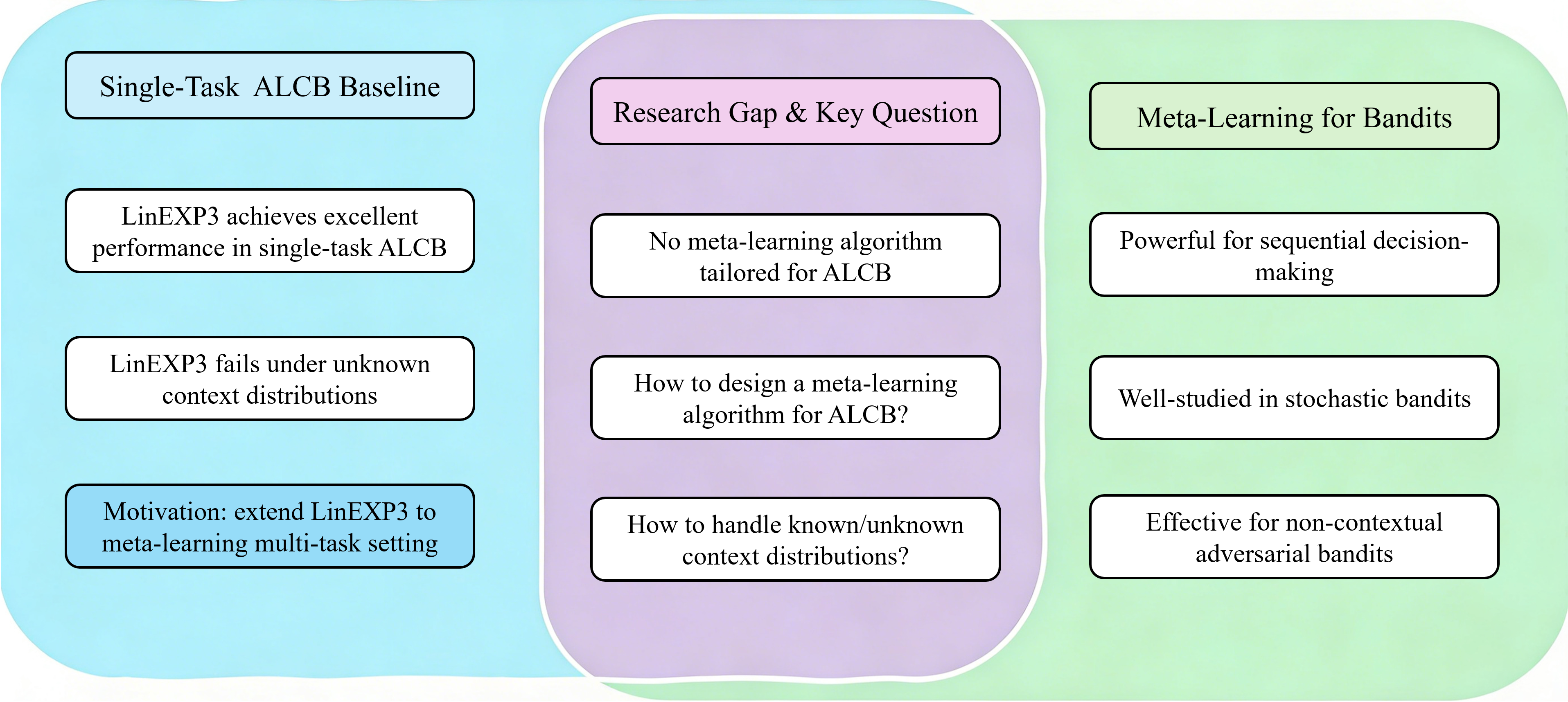}
    \caption{Relationship between Meta-LinEXP3 and closely related bandit settings. The left column summarizes the single-task LinEXP3 starting point, the middle column gives the ALCB meta-learning problem studied here, and the right column lists related meta-bandit settings.}
    \label{fig:research-gap}
\end{figure}

\section{Preliminaries}

\subsection{Notation}\label{sec:notation}
For any positive integer $q$, let $[q]=\{1,\ldots,q\}$. The global problem dimensions are the number of tasks $m\ge1$, rounds per task $n\ge1$, actions per round $k\ge2$, and context dimension $d\ge1$. Here, $T=mn$ is the total number of task--round pairs. We use $s\in[m]$, $t\in[n]$, and $a\in[k]$ for task, round, and action indices, respectively. Indices $i$ and $j$ refer to earlier tasks or rounds when their ranges are displayed. Bold lowercase and uppercase letters denote vectors and matrices, calligraphic letters denote sets, distributions, sigma-fields, or tensors as specified locally, and ordinary letters denote scalars unless stated otherwise. Superscript $T$ is transpose and $\dagger$ is the Moore--Penrose pseudoinverse. We write $\boldsymbol0_d$, $\boldsymbol1_d$, $\boldsymbol I_d$, and $\boldsymbol e_a$ for the $d$-dimensional zero vector, the $d$-dimensional all-ones vector, the $d\times d$ identity matrix, and the relevant standard basis vector, with the dimension omitted only when it is fixed by context.

The notation $\langle\cdot,\cdot\rangle$ denotes the Euclidean inner product. The symbol $\|\cdot\|$ is the Euclidean norm for vectors and the operator norm for matrices, $\|\cdot\|_F$ is the Frobenius norm, and $\|\boldsymbol x\|_{\boldsymbol A}=\sqrt{\boldsymbol x^T\boldsymbol A\boldsymbol x}$ for a positive semidefinite matrix $\boldsymbol A$. For symmetric matrices, $\boldsymbol A\succeq\boldsymbol B$ means that $\boldsymbol A-\boldsymbol B$ is positive semidefinite. We use $\lambda_{\min}(\boldsymbol A)$, $\operatorname{tr}(\boldsymbol A)$, $\operatorname{range}(\boldsymbol A)$, and $\operatorname{rank}(\boldsymbol A)$ for the minimum eigenvalue, trace, column space, and rank. The operators $\operatorname{Cov}$, $\operatorname{diag}$, and $\operatorname{chol}$ denote covariance, diagonalization of a vector, and a Cholesky factor, respectively. For a set $S$, $\operatorname{span}(S)$ is its linear span, and $\operatorname{supp}(\mathcal D)$ is the topological support of distribution $\mathcal D$.

For a subspace $\mathcal V\subseteq\mathbb R^d$, $\boldsymbol P_{\mathcal V}$ is the orthogonal projector and $\boldsymbol A|_{\mathcal V}$ is the restriction of $\boldsymbol A$ to $\mathcal V$. For $\boldsymbol b\in\mathbb R^d$, set $\boldsymbol P_{\boldsymbol b}=\boldsymbol b\boldsymbol b^T/\|\boldsymbol b\|^2$ when $\boldsymbol b\ne\boldsymbol0_d$ and $\boldsymbol P_{\boldsymbol0_d}=\boldsymbol0_{d\times d}$. We write $x_+=\max\{x,0\}$ and $\mathbf1\{E\}$ for the positive part of $x$ and the indicator of event $E$. A subscripted symbol such as $\mathbf1_L$ denotes the incidence vector of a finite set $L$ in the stated ambient space. All $\operatorname*{arg\,min}$ expressions use a fixed measurable tie-breaking rule.

For an event $E$, $E^c$, $\mathbb P(E)$, and $\mathbb E$ denote its complement, probability, and expectation. A subscript on $\mathbb E$ specifies the variable being integrated out. The notation $\sigma(\cdot)$ denotes a generated sigma-field and $\mathcal G\vee\mathcal H$ its join. Finally, $\mathcal O(\cdot)$, $\widetilde{\mathcal O}(\cdot)$, and $o(\cdot)$ have their standard asymptotic meanings, with $\widetilde{\mathcal O}$ suppressing logarithmic factors.

\subsection{Online-within-Online ALCB Problem}\label{sec2_1}
We consider $m$ sequential ALCB tasks. In task $s\in[m]$, the learner and an adversarial environment interact for $n$ rounds. At round $t\in[n]$:
\begin{enumerate}
\item The environment selects a loss vector $\boldsymbol{\theta}_{s,t} \in \mathbb{R}^d$.
\item Independently of $\boldsymbol{\theta}_{s,t}$, the environment samples $k$ context vectors from the context distribution $\mathcal{D}$ to form  $\mathcal{B}_{s,t} = \left\{ \boldsymbol{b}_{s,t,a} \right\}_{a=1}^{k} \subset \mathbb{R}^d$, and fully reveals $\mathcal{B}_{s,t}$ to the learner.
\item Conditioned on the observed set $\mathcal{B}_{s,t}$, the learner selects an action $A_{s,t} \in [k]$, allowing randomized decision rules.
\item The learner incurs a loss corresponding to the selected action, which is given by $\ell_{s,t,A_{s,t}} = \langle \boldsymbol{b}_{s,t,A_{s,t}}, \boldsymbol{\theta}_{s,t} \rangle$.
\end{enumerate}

Here $\mathcal D$ is the common context distribution. We write
$\boldsymbol b\sim\mathcal D$ for a generic context and
$\mathcal B=\{\boldsymbol b_a\}_{a=1}^k\sim\mathcal D^k$ for an independent ordered context set.

We measure performance by expected pseudo-regret relative to a policy fixed independently of the realized context trajectory. Let $\Pi$ be a prescribed class of measurable mappings from a context set $\mathcal{B}$ to an action in $[k]$. For $\boldsymbol w\in\mathbb R^d$, let $\pi_{\boldsymbol w}(\mathcal B)$ be the tie-broken minimizer of $\langle\boldsymbol b_a,\boldsymbol w\rangle$ over $a\in[k]$. For every fixed $\pi\in\Pi$, define
\begin{equation}\label{eq:policy-regret}
 \mathrm{Reg}_s(\pi)=
 \mathbb{E}\!\left[\sum_{t=1}^n
 \left\langle
 \boldsymbol b_{s,t,A_{s,t}}-
 \boldsymbol b_{s,t,\pi(\mathcal B_{s,t})},
 \boldsymbol\theta_{s,t}
 \right\rangle\right],
 \qquad
 \mathrm{Reg}_s=\sup_{\pi\in\Pi}\mathrm{Reg}_s(\pi).
\end{equation}
The supremum is taken outside the expectation. This distinction prevents the comparator from being chosen after observing the realized context trajectory: an unrestricted post-hoc mapping could otherwise memorize the observed sets and behave as a dynamic oracle. We also write $\mathrm{Reg}_{1:m}=\sum_{s=1}^m\mathrm{Reg}_s$.

To state the non-anticipation condition precisely, let the pre-context history be
\begin{equation}\label{eq:pre-filtration}
\mathcal H_{s,t}=\sigma\!\left(
\mathcal F_{1:s-1},
\mathcal B_{s,j},A_{s,j},\ell_{s,j,A_{s,j}}:j<t
\right),
\end{equation}
where $\mathcal F_{1:s-1}$ is generated by all observations from tasks $1,\ldots,s-1$, with $\mathcal F_{1:0}$ denoting the trivial sigma-field. The adversary may adapt to $\mathcal H_{s,t}$, and $\boldsymbol\theta_{s,t}$ is $\mathcal H_{s,t}$-measurable. Conditional on $\mathcal H_{s,t}$, the set $\mathcal B_{s,t}\sim\mathcal D^k$ is fresh and independent of $\boldsymbol\theta_{s,t}$. Before this set arrives, the learner fixes an $\mathcal H_{s,t}$-measurable random Markov kernel $\mathcal B\mapsto p_{s,t}(\cdot\mid\mathcal B)$ from $(\mathbb R^d)^k$ to the probability simplex on $[k]$. After the set is revealed, the information is $\mathcal F_{s,t}=\mathcal H_{s,t}\vee\sigma(\mathcal B_{s,t})$, and $A_{s,t}$ is sampled from $p_{s,t}(\cdot\mid\mathcal B_{s,t})$. We write $\mathbb E^{\mathrm{pre}}_{s,t}[\cdot]=\mathbb E[\cdot\mid\mathcal H_{s,t}]$ and $\mathbb E_{s,t}[\cdot]=\mathbb E[\cdot\mid\mathcal F_{s,t}]$. All unbiasedness statements below are made at the pre-context conditioning time.

\subsection{The LinEXP3 Algorithm}
We use LinEXP3 \cite{ALCB} as the inner learner. At each round, LinEXP3 forms exponential weights from the current context set and the loss estimates accumulated so far, then mixes the resulting distribution with uniform exploration of total mass $\gamma$. After an action is sampled and its loss is observed, the corresponding loss vector estimate is incorporated into subsequent updates.

\section{The Meta-LinEXP3 Algorithm}
Meta-LinEXP3 augments the inner LinEXP3 learner with a task-level prior. Before task $s$, the outer learner maps completed-task summaries to a predictable vector $\boldsymbol h_s$ and keeps the induced prior fixed throughout the task. This design removes within-task prior-variation terms from the regret analysis. The analysis extends to another loss estimator only after verifying the corresponding conditions in Section~\ref{sec:theory}.

For the shared context distribution, define
\begin{equation}\label{eq:context-geometry}
 \bar{\boldsymbol b}=\mathbb E[\boldsymbol b],\qquad
 \boldsymbol C=\operatorname{Cov}(\boldsymbol b),\qquad
 \mathcal U=\operatorname{span}\{\boldsymbol b-\boldsymbol b':
 \boldsymbol b,\boldsymbol b'\in\operatorname{supp}(\mathcal D)\}
 =\operatorname{range}(\boldsymbol C).
\end{equation}
All feasible context differences lie in $\mathcal U$, so only this subspace can affect pairwise action comparisons.

\subsection{Overall Algorithm Structure}
Before task $s$, the learner constructs $\boldsymbol h_s$ from summaries of tasks $1,\ldots,s-1$. During the task, the softmax prior induced by $\boldsymbol h_s$ is combined with the cumulative within-task loss estimate after each context set is observed. At the end of the task, the mean loss estimate and its regularized direction are stored for future transfer. Algorithm~\ref{alg1} gives the common outer--inner procedure. Estimator-specific inputs are described in Section~\ref{sec:estimators}.

\begin{algorithm}[t]
\caption{Meta-LinEXP3}
\begin{algorithmic}[1]
\REQUIRE $m,n,k$, task-dependent parameters $\eta_s>0$, $\mu_s\ge0$, and $\gamma_s\in(0,1)$, direction ridge $\varepsilon_\Theta>0$, cosine smoothing $\tau>0$, and one estimator with its inputs from Section~\ref{sec:estimators}.
\STATE Set $\boldsymbol h_1=\boldsymbol0_d$ and $\widehat{\boldsymbol\Theta}_0=\widehat{\boldsymbol v}_0=\boldsymbol0_d$.

\FOR{ task $s=1,2,\ldots,m$}
  \STATE If $s\ge2$, compute the fixed vector $\boldsymbol h_s$ using PCRW \eqref{eqw1} or uniform aggregation \eqref{eqw2}. Retain $\boldsymbol h_1=\boldsymbol0_d$. Then set
  \[
  r_s(a\mid\mathcal B)=
  \frac{\exp[-\mu_s\langle\boldsymbol b_a,\boldsymbol h_s\rangle]}
  {\sum_{a'\in[k]}\exp[-\mu_s\langle\boldsymbol b_{a'},\boldsymbol h_s\rangle]}.
  \]
  \STATE Set $\widehat{\boldsymbol\theta}_{s,0}=\boldsymbol0_d$.
  \FOR{ round $t=1,2,\ldots,n$ }
    \STATE Observe $\mathcal B_{s,t}$ and form
    \[
    q_{s,t}(a\mid\mathcal B_{s,t})
      \propto r_s(a\mid\mathcal B_{s,t})
      \exp\!\left[-\eta_s\left\langle\boldsymbol b_{s,t,a},
      \sum_{j=0}^{t-1}\widehat{\boldsymbol\theta}_{s,j}\right\rangle\right].
    \]
    \STATE Sample $A_{s,t}$ from $p_{s,t}(a\mid\mathcal B_{s,t})=(1-\gamma_s)q_{s,t}(a\mid\mathcal B_{s,t})+\gamma_s/k$.
    \STATE Observe the loss of the selected action $\ell_{s,t,A_{s,t}}$, and compute the loss estimator $\widehat{\boldsymbol{\theta }}_{s,t}$.
  \ENDFOR
  \STATE Store $\widehat{\boldsymbol\Theta}_s=n^{-1}\sum_{t=1}^n\widehat{\boldsymbol\theta}_{s,t}$ and $\widehat{\boldsymbol v}_s=\boldsymbol P_{\mathcal U}\widehat{\boldsymbol\Theta}_s/\max\{\|\boldsymbol P_{\mathcal U}\widehat{\boldsymbol\Theta}_s\|,\varepsilon_\Theta\}$ when $\mathcal U$ is known. Otherwise, use the same formula without $\boldsymbol P_{\mathcal U}$.
\ENDFOR
\end{algorithmic}
\label{alg1}
\end{algorithm}

PC-KDE additionally requires $\mathcal D$ and an exact selected-moment oracle. PRME requires $L$, $\lambda$, $d$, $T=mn$, $\delta$, and spectral clipping as specified in Sections~\ref{sec:estimators} and \ref{sec:unknown-theory}. LPE has no distributional input.

For cross-task transfer, represent task $s$ by its mean loss vector
$\boldsymbol\Theta_s:=n^{-1}\sum_{t=1}^n\boldsymbol\theta_{s,t}$. To interpret task similarity in this subsection, we consider oblivious tasks, for which $\boldsymbol\Theta_s$ is fixed before the task's context trajectory is generated. By \eqref{eq:context-geometry}, only the projection onto $\mathcal U$ affects action differences. In the absence of ties, the optimal fixed policy is therefore determined by the positive ray of $\boldsymbol P_{\mathcal U}\boldsymbol\Theta_s$. Any component in $\mathcal U^\perp$ adds the same offset to every action. Euclidean distance in the ambient space can consequently be misleading. If $\boldsymbol P_{\mathcal U}\boldsymbol{\Theta }_i = M \boldsymbol P_{\mathcal U}\boldsymbol{\Theta }_s$ for $M>0$ and $\| \boldsymbol P_{\mathcal U}\boldsymbol{\Theta }_s \| \neq 0$, the two tasks induce exactly the same optimal action selection policy even though their full-space distance can be arbitrarily large. For $M<0$, by contrast, the loss ordering is generally reversed.

We therefore measure decision-relevant similarity by the signed cosine of the projected task means,
\[
 \frac{\langle \boldsymbol P_{\mathcal U}\boldsymbol\Theta_i,
 \boldsymbol P_{\mathcal U}\boldsymbol\Theta_s\rangle}
 {\|\boldsymbol P_{\mathcal U}\boldsymbol\Theta_i\|\,
  \|\boldsymbol P_{\mathcal U}\boldsymbol\Theta_s\|},
\]
and set the value to zero when either projection vanishes. Values near $1$, near $0$, and below $0$ correspond to aligned, weakly related, and potentially adverse transfer, respectively. PCRW in \eqref{eqw1} uses the positive part of the empirical cosine. PC-KDE summaries already lie in $\mathcal U$. For other estimators, a full-space cosine should be viewed as a heuristic unless the summaries are first projected onto a decision-relevant subspace.

\subsection{Loss Estimators}\label{sec:estimators}
Because the learner selects an action after observing the entire context set, the selected context need not follow the marginal distribution $\mathcal D$. Correcting this selection effect requires either inverse-propensity weighting or a design matrix induced by the current policy. ROBUST LinEXP3 uses the former \cite{ALCB}, whereas the known-distribution construction of \cite{Liu} uses policy-induced moments. Our estimator retains the centered form but centers with respect to the distribution of the \emph{selected} context.

Using the context mean in \eqref{eq:context-geometry}, write the covariance as
\begin{equation}\label{eq:marginal-covariance}
 \boldsymbol C=\mathbb E[(\boldsymbol b-\bar{\boldsymbol b})(\boldsymbol b-\bar{\boldsymbol b})^T],
 \qquad \operatorname{range}(\boldsymbol C)=\mathcal U,
\end{equation}
where $\mathcal U$ is the decision-relevant subspace in \eqref{eq:context-geometry}.
For a known context distribution and the current policy $p_{s,t}$, define its policy-induced selected mean and covariance before $\mathcal B_{s,t}$ arrives:
\begin{align}
 \boldsymbol x_{s,t}
 &=\mathbb E_{\mathcal B\sim\mathcal D^k}
   \left[\sum_{a=1}^k p_{s,t}(a\mid\mathcal B)\boldsymbol b_a\right],
 \label{eq:selected-mean}\\
 \boldsymbol H_{s,t}
 &=\mathbb E_{\mathcal B\sim\mathcal D^k}
   \left[\sum_{a=1}^k p_{s,t}(a\mid\mathcal B)
   (\boldsymbol b_a-\boldsymbol x_{s,t})(\boldsymbol b_a-\boldsymbol x_{s,t})^T\right].
 \label{eq:selected-covariance}
\end{align}
Both quantities are $\mathcal H_{s,t}$-measurable. We call
\begin{equation}\label{eq17}
 \widehat{\boldsymbol\theta}^{\mathrm{PC}}_{s,t}
 =\boldsymbol H_{s,t}^{\dagger}
 (\boldsymbol b_{s,t,A_{s,t}}-\boldsymbol x_{s,t})
 \ell_{s,t,A_{s,t}}
\end{equation}
the policy-centered known-distribution estimator (PC-KDE), where $\dagger$ denotes the Moore--Penrose pseudoinverse. The uniform component of $p_{s,t}$ implies that $\boldsymbol H_{s,t}$ is positive definite on $\mathcal U$. Moreover,
\begin{equation}\label{eq:pc-unbiased}
 \mathbb E^{\mathrm{pre}}_{s,t}
 [\widehat{\boldsymbol\theta}^{\mathrm{PC}}_{s,t}]
 =\boldsymbol P_{\mathcal U}\boldsymbol\theta_{s,t},
\end{equation}
where $\boldsymbol P_{\mathcal U}$ is the orthogonal projector onto $\mathcal U$. Hence PC-KDE is exactly unbiased when $\boldsymbol C$ is nonsingular. More generally, since every context difference lies in $\mathcal U$ almost surely, \eqref{eq:pc-unbiased} is sufficient to recover all loss differences exactly, which is all that the regret analysis requires. The same affine-subspace formulation covers fixed-cardinality context vectors, whose centered covariance is necessarily singular in the ambient space.

PC-KDE requires exact evaluation of \eqref{eq:selected-mean}--\eqref{eq:selected-covariance} under the current adaptive policy. Its computational cost therefore depends on the distribution and on the moment oracle. Knowledge of $\mathcal D$ alone is not sufficient to evaluate these quantities at no cost. Approximating them numerically or by Monte Carlo would introduce an additional error term that is outside the present analysis.

When $\mathcal D$ is unknown, the policy-induced expectations in \eqref{eq:selected-mean}--\eqref{eq:selected-covariance} are not directly available. Define the raw second moment
$\boldsymbol S=\mathbb E[\boldsymbol b\boldsymbol b^T]$ and use inverse propensities. Let
\begin{equation}\label{eq:past-sample-size}
 N_{s,t}=k\bigl(n(s-1)+t-1\bigr)
\end{equation}
be the number of context vectors observed strictly before the current set. Let $<$ on task--round pairs denote lexicographic order, and define
\begin{align}
 \overline{\boldsymbol S}_{s,t}
 &=\begin{cases}
 \boldsymbol0_{d\times d}, & N_{s,t}=0,\\
 \displaystyle\frac{1}{N_{s,t}}
 \sum_{(i,j,a):\,(i,j)<(s,t)}
 \boldsymbol b_{i,j,a}\boldsymbol b_{i,j,a}^T, & N_{s,t}\ge1,
 \end{cases}\label{eq:lagged-moment}\\
 \widehat{\boldsymbol S}_{s,t}
 &=\overline{\boldsymbol S}_{s,t}+\beta_{s,t}\boldsymbol I_d,\\
 \widetilde{\boldsymbol S}_{s,t}
 &=\operatorname{Clip}_{\lambda}
 (\widehat{\boldsymbol S}_{s,t}),
\end{align}
where $\beta_{s,t}\ge0$ is a pre-context measurable ridge specified in Section~\ref{sec:unknown-theory}. Here $\lambda>0$ is a known lower bound on $\lambda_{\min}(\boldsymbol S)$, and $\operatorname{Clip}_{\lambda}(\boldsymbol A)$ replaces every eigenvalue of a symmetric matrix $\boldsymbol A$ below $\lambda$ by $\lambda$ while preserving its eigenvectors. This clipping is inactive on the simultaneous concentration event used in the analysis, but keeps the estimator and its range controlled on the failure event. The past-only regularized moment estimator (PRME) is
\begin{equation}\label{eq18}
 \widehat{\boldsymbol\theta}^{\mathrm{PR}}_{s,t}
 =\frac{1}{k\,p_{s,t}(A_{s,t}\mid\mathcal B_{s,t})}
 \widetilde{\boldsymbol S}_{s,t}^{-1}
 \boldsymbol b_{s,t,A_{s,t}}\ell_{s,t,A_{s,t}}.
\end{equation}
The factor $1/(kp_{s,t})$ removes action selection bias, while the use of strictly past contexts makes $\widehat{\boldsymbol S}_{s,t}$ measurable before the current set arrives. Its exact conditional mean is
\begin{equation}\label{eq:prme-mean}
 \mathbb E^{\mathrm{pre}}_{s,t}
 [\widehat{\boldsymbol\theta}^{\mathrm{PR}}_{s,t}]
 =\widetilde{\boldsymbol S}_{s,t}^{-1}\boldsymbol S\boldsymbol\theta_{s,t}.
\end{equation}
Thus PRME carries an explicit finite-sample regularization error. We make no claim of exact or asymptotic unbiasedness without an accompanying rate.

\paragraph{Computational cost of PRME.}
PRME maintains a single $d\times d$ streaming sufficient statistic and never replays earlier context sets, requiring $\mathcal O(d^2)$ memory. With direct inversion and spectral clipping, each update costs $\mathcal O(d^3+kd^2)$. Its general $n^{2/3}$ rate is weaker than the $\widetilde{\mathcal O}(\sqrt n)$ single-task rates of \cite{Liu,vanErven2025}. In return, the policy-independent moment makes cross-task context reuse and transfer complexity explicit in the bound.

For resource-constrained settings, we also consider the lightweight projection estimator (LPE):
\begin{equation}\label{eq999}
 \widehat{\boldsymbol\theta}^{\mathrm{LP}}_{s,t}
 =\frac{\boldsymbol b_{s,t,A_{s,t}}}
 {\|\boldsymbol b_{s,t,A_{s,t}}\|^2}
 \ell_{s,t,A_{s,t}},
\end{equation}
where the estimator is set to zero for a zero context. Its per-round complexity is $\mathcal O(d)$ and, in the noiseless linear model, it obeys the pointwise identity
\begin{equation}\label{eq:lpe-projection}
 \widehat{\boldsymbol\theta}^{\mathrm{LP}}_{s,t}
 =\boldsymbol P_{\boldsymbol b_{s,t,A_{s,t}}}\boldsymbol\theta_{s,t}.
\end{equation}
Here $\boldsymbol P_{\boldsymbol b}$ is the projector defined in Section~\ref{sec:notation}. Thus LPE estimates only the component of the loss vector along the selected context. Appendix~\ref{app:proofs} gives the resulting projection-bias term and regret decomposition.

\subsection{Task Weightings}

We use \emph{positive-cosine retrieval weighting} (PCRW) to downweight anti-aligned historical directions. The resulting aggregate is computed before each task and then held fixed. The regularized directions $\widehat{\boldsymbol v}_i$ in Algorithm~\ref{alg1} satisfy $\|\widehat{\boldsymbol v}_i\|\le1$. For $s\ge2$, use the most recent completed task as the predictable query and define
\begin{align}
 a_{s,i}&=\langle\widehat{\boldsymbol v}_i,
                 \widehat{\boldsymbol v}_{s-1}\rangle_+,
 \qquad i<s,\notag\\
 c_{s,i}^{\mathrm{cos}}
 &=\frac{a_{s,i}+\tau/(s-1)}
 {\sum_{j=1}^{s-1}a_{s,j}+\tau},
 \qquad
 \boldsymbol h_s^{\mathrm{cos}}
 =\sum_{i=1}^{s-1}c_{s,i}^{\mathrm{cos}}
 \widehat{\boldsymbol v}_i .
 \label{eqw1}
\end{align}
Here $\tau>0$ is a smoothing parameter. The coefficients are nonnegative and sum to one, so $\|\boldsymbol h_s^{\mathrm{cos}}\|\le1$. Anti-aligned tasks receive only the smoothing mass, zero summaries remain well defined, and the ridge $\varepsilon_\Theta$ prevents division by a near-zero norm. Since $\boldsymbol h_s^{\mathrm{cos}}$ is fixed before task $s$, it contributes no within-task path variation.

A simpler alternative is uniform convex aggregation,
\begin{equation}\label{eqw2}
 c_{s,i}^{\mathrm{uni}}=\frac1{s-1},\qquad
 \boldsymbol h_s^{\mathrm{uni}}
 =\frac1{s-1}\sum_{i=1}^{s-1}\widehat{\boldsymbol v}_i,
 \qquad s\ge2,
\end{equation}
with $\boldsymbol h_1=\boldsymbol0_d$. It is also fixed within a task and has norm at most one.

PCRW is intended for task streams with sequentially persistent structure. When no such structure is available, setting $\boldsymbol h_s=\boldsymbol0_d$ recovers the zero-prior setting and gives $\Gamma_s(\pi)=\log k$. The transfer theorem below applies to any predictable prior that is fixed within a task, with its effect summarized by the transfer complexity. Appendix~\ref{app:proofs} decomposes the PCRW tracking error into retrieval dispersion, local task-direction variation, and direction-estimation error.

\section{Theoretical Analysis}\label{sec:theory}

We first state the main regret guarantees and defer the changing-prior lemmas, exact bias decompositions, alternative fast-rate conditions, Gaussian calculations, LPE results, and complete proofs to Appendix~\ref{app:proofs}. For a fixed task $s$, write $\eta=\eta_s$ and $\gamma=\gamma_s$.

\subsection{Common Assumptions and Transfer Complexity}

Define the geometric quantities
\begin{equation}
L=\sup_{\boldsymbol b\in\operatorname{supp}(\mathcal D)}
\|\boldsymbol b\|,
\qquad
\Delta=
\sup_{\boldsymbol b,\boldsymbol b'\in\operatorname{supp}(\mathcal D)}
\|\boldsymbol b-\boldsymbol b'\|.
\end{equation}

\begin{assumption}[Bounded loss geometry]\label{ass0}
The context distribution has bounded support.
For every $s,t$, $\|\boldsymbol\theta_{s,t}\|\le R$ and
$|\langle\boldsymbol b,\boldsymbol\theta_{s,t}\rangle|\le Y$ almost surely. The within-set loss range is bounded by $G$:
\begin{equation}
 \max_a\langle\boldsymbol b_a,\boldsymbol\theta_{s,t}\rangle
 -\min_a\langle\boldsymbol b_a,\boldsymbol\theta_{s,t}\rangle\le G.
\end{equation}
\end{assumption}

The choices $Y=LR$ and $G=R\Delta\le2Y$ satisfy Assumption~\ref{ass0}.

Given the pre-task vector $\boldsymbol h_s$ and prior concentration parameter $\mu_s$, define the task prior by
\begin{equation}\label{eq:fixed-meta-prior}
 r_s(a\mid\mathcal B)=
 \frac{\exp[-\mu_s\langle\boldsymbol b_a,\boldsymbol h_s\rangle]}
 {\sum_{a'\in[k]}\exp[-\mu_s\langle\boldsymbol b_{a'},\boldsymbol h_s\rangle]}
\end{equation}

\begin{assumption}[Predictable fixed task prior]\label{ass1}
Before task $s$ begins, the learner chooses an $\mathcal H_{s,1}$-measurable vector
$\boldsymbol h_s$ with $\|\boldsymbol h_s\|\le1$ and a prior concentration parameter $\mu_s\ge0$.
The resulting task prior is held fixed throughout task $s$.
\end{assumption}

When the prior parameters are displayed explicitly, $r_{\mu,\boldsymbol h}$ denotes the softmax kernel on the right-hand side of \eqref{eq:fixed-meta-prior} with prior concentration parameter $\mu$ and vector $\boldsymbol h$.

Let $\mathcal B^0=\{\boldsymbol b_a^0\}_{a=1}^k\sim\mathcal D^k$ be independent of the interaction. For a fixed comparator $\pi\in\Pi$, define the transfer complexity
\begin{equation}\label{eq:fixed-transfer-complexity}
 \Gamma_s(\pi)
 =\mathbb E\!\left[-\log
 r_s(\pi(\mathcal B^0)\mid\mathcal B^0)\right].
\end{equation}
Fixing the prior within a task makes its path-variation charge exactly zero. The universal fallback
\begin{equation}\label{eq:transfer-uniform-fallback}
 \Gamma_s(\pi)\le\log k+\mu_s\Delta
\end{equation}
holds for every unit-norm prior vector, while $\boldsymbol h_s=\boldsymbol0_d$ gives
$\Gamma_s(\pi)=\log k$. For $\zeta\in(0,1]$, set
\begin{equation}\label{eq:psi-c}
 \psi(\zeta)=\frac{e^\zeta-1-\zeta}{\zeta^2}.
\end{equation}

\subsection{Known Context Distribution: PC-KDE}

Let $\boldsymbol C$ and $\mathcal U$ be defined in \eqref{eq:marginal-covariance}, and write
\begin{equation}
 d_{\mathcal U}=\operatorname{rank}(\boldsymbol C),\qquad
 \lambda_+=\lambda_{\min}(\boldsymbol C|_{\mathcal U}),\qquad
 \kappa_{\mathcal U}=\frac{\Delta^2}{\lambda_+}.
\end{equation}
If $d_{\mathcal U}=0$, every action has the same context almost surely and regret is zero.

\begin{theorem}[PC-KDE bound]\label{theo01}
Suppose Assumptions~\ref{ass0}--\ref{ass1} hold, $d_{\mathcal U}>0$, and PC-KDE in \eqref{eq17} is used. Then
\begin{equation}\label{eq:pc-parameter-bias}
 \mathbb E_{s,t}^{\mathrm{pre}}
 [\widehat{\boldsymbol\theta}^{\mathrm{PC}}_{s,t}]
 =\boldsymbol P_{\mathcal U}\boldsymbol\theta_{s,t}.
\end{equation}
For every fixed $\pi\in\Pi$, if
\begin{equation}\label{eq:pc-stability}
 \eta Y\kappa_{\mathcal U}\le\zeta\gamma,
\end{equation}
then
\begin{equation}\label{eq:pc-regret-simple}
 \mathrm{Reg}_s(\pi)
 \le\frac{(1-\gamma)\Gamma_s(\pi)}{\eta}
 +\psi(\zeta)\eta d_{\mathcal U}Y^2n+\gamma Gn.
\end{equation}
\end{theorem}

Since all context differences lie in $\mathcal U$, identity~\eqref{eq:pc-parameter-bias} is sufficient for the regret analysis.
For the tuned bound, define
\begin{equation}
 A_\zeta=\psi(\zeta)d_{\mathcal U}Y^2+
 \frac{Y\kappa_{\mathcal U}G}{\zeta}.
\end{equation}

\begin{corollary}[Tuned PC-KDE bound]\label{cor1}
Suppose a deterministic $\bar\Gamma_s>0$ satisfies
$\Gamma_s(\pi)\le\bar\Gamma_s$ uniformly over the candidate tuning pairs. Choose
\begin{equation}\label{eq:pc-tuning}
 \eta_s=\sqrt{\frac{\bar\Gamma_s}{nA_\zeta}},
 \qquad
 \gamma_s=\frac{\eta_sY\kappa_{\mathcal U}}{\zeta}.
\end{equation}
If $\gamma_s\le1/2$, then
\begin{equation}\label{eq:pc-sqrt-rate}
 \mathrm{Reg}_s(\pi)\le
 2\sqrt{nA_\zeta\bar\Gamma_s}.
\end{equation}
\end{corollary}

When the tuned exploration probability exceeds $1/2$, the untuned bound~\eqref{eq:pc-regret-simple} or the trivial bound $Gn$ can be used instead.

\subsection{Unknown Context Distribution: PRME}\label{sec:unknown-theory}

For the raw second moment $\boldsymbol S$ introduced in Section~\ref{sec:estimators}, assume
\begin{equation}\label{eq:raw-moment-lower}
 \boldsymbol S\succeq\lambda\boldsymbol I_d,\qquad
 \kappa=\frac{L^2}{\lambda},\qquad
 \chi=\min\{k,\kappa\}.
\end{equation}
For $T=mn$ and $\delta\in(0,1)$, define
\begin{equation}\label{eq:bernstein-radius}
 \Lambda=\log\frac{2dT}{\delta},\qquad
 \xi_N=L^2\!\left[
 \frac{\Lambda}{3N}+
 \sqrt{\frac{2\Lambda}{N}+\frac{\Lambda^2}{9N^2}}
 \right]\ (N\ge1),\qquad
 \xi_0=L^2.
\end{equation}
PRME uses $\beta_{s,t}=\xi_{N_{s,t}}$ in \eqref{eq:lagged-moment}.

\begin{theorem}[Tuned PRME bound]\label{theo:prme-core}
Suppose Assumptions~\ref{ass0}--\ref{ass1} and
\eqref{eq:raw-moment-lower} hold. Let
$\Gamma_s(\pi)\le\bar\Gamma_s$ for a deterministic $\bar\Gamma_s>0$, uniformly over candidate tuning pairs, and let $G>0$. Set
\begin{align}
 \gamma_s&=
 \left(\frac{\psi(1)\chi dY^2\bar\Gamma_s}
 {G^2n}\right)^{1/3},\notag\\
 \eta_s&=
 \frac{\bar\Gamma_s^{2/3}}
 {\{\psi(1)\chi d\}^{1/3}Y^{2/3}G^{1/3}n^{2/3}}.
 \label{eq:prme-tuning}
\end{align}
If $\gamma_s\le1/2$ and
\begin{equation}\label{eq:prme-feasibility}
 \kappa^3\bar\Gamma_s G
 \le\psi(1)^2(\chi d)^2Yn,
\end{equation}
then PRME satisfies
\begin{align}
 \mathrm{Reg}_s(\pi)
 \le{}&
 3\{\psi(1)\chi d\bar\Gamma_sY^2G\}^{1/3}n^{2/3}
 +4R\sqrt{\frac{\chi}{\lambda}}
 \sum_{t=1}^n\xi_{N_{s,t}}
 +f_s,
 \label{eq:prme-rate}
\end{align}
where the failure-event remainder is bounded by
\begin{equation}\label{eq:failure-remainder}
 f_s
 \le\left[
 \psi(1)\eta_s\left(\frac{Y\kappa}{\gamma_s}\right)^2
 +2LR\sqrt\chi(\kappa+1)\right]\frac{\delta}{m}.
\end{equation}
\end{theorem}

In lexicographic task--round order, the accumulated moment radius also satisfies
\begin{equation}\label{eq:sum-xi}
 \sum_{u=1}^{T}\xi_{k(u-1)}
 \le L^2\!\left[
 1+2\sqrt{\frac{2\Lambda(T-1)}{k}}
 +\frac{2\Lambda}{3k}\{1+\log(T-1)\}
 \right]
\end{equation}
for $T\ge2$, while the sum equals $L^2$ for $T=1$.

\subsection{Meta transfer analysis}\label{sec:meta-across-tasks}

For an oblivious task with
$\boldsymbol P_{\mathcal U}\boldsymbol\Theta_s\ne\boldsymbol0_d$, let
\begin{equation}\label{eq:true-task-direction}
 \boldsymbol v_s=
 \frac{\boldsymbol P_{\mathcal U}\boldsymbol\Theta_s}
 {\|\boldsymbol P_{\mathcal U}\boldsymbol\Theta_s\|},
 \qquad
 \pi_s(\mathcal B)=\pi_{\boldsymbol v_s}(\mathcal B),
\end{equation}
using the tie-breaking convention from Section~\ref{sec:notation}.

For the cross-task analysis, define the augmented pre-task sigma-field
\begin{equation}\label{eq:augmented-pre-task}
 \mathcal G_s:=\mathcal H_{s,1}\vee\sigma(\boldsymbol\Theta_s).
\end{equation}
The learner does not observe $\mathcal G_s$: the prior $\boldsymbol h_s$ remains
$\mathcal H_{s,1}$-measurable, and $\mathcal G_s$ is used only for the cross-task
conditional analysis. Let $\mathcal B^0=\{\boldsymbol b_a^0\}_{a=1}^k\sim\mathcal D^k$
be a fresh ghost set independent of $\mathcal G_s$. Then $\boldsymbol v_s$ and the
prior-error event in Theorem~\ref{theo:general-margin} are $\mathcal G_s$-measurable,
while the ghost contexts remain independent. Define the random minimum gap by
\begin{equation}\label{eq:general-gap}
 \operatorname{gap}_s(\mathcal B^0)
 =\min_{a\ne\pi_s(\mathcal B^0)}
 \langle\boldsymbol b_a^0-\boldsymbol b_{\pi_s(\mathcal B^0)}^0,
 \boldsymbol v_s\rangle
\end{equation}

\begin{assumption}[Margin condition]\label{ass:general-margin}
Each task is oblivious: its entire loss sequence is fixed before the task begins. Assume
$\boldsymbol P_{\mathcal U}\boldsymbol\Theta_s\ne\boldsymbol0_d$, $\pi_s\in\Pi$, and
$\Delta>0$. There exists a nondecreasing function
$F:[0,\infty)\to[0,1]$ with
\begin{equation}\label{eq:general-margin-modulus}
 \lim_{x\to 0^+}F(x)=0
\end{equation}
such that, conditionally on the augmented pre-task information,
\begin{equation}\label{eq:general-margin-cdf}
 \mathbb P\!\left(
 \operatorname{gap}_s(\mathcal B^0)\le x
 \mid\mathcal G_s\right)\le F(x),\qquad x>0.
\end{equation}
\end{assumption}

For $\mu>0$, define the margin function
\begin{equation}\label{eq:margin-laplace-envelope}
 \Phi_F(\mu)=\mu\int_0^\infty e^{-\mu x}F(x)\,dx
 =\int_0^\infty e^{-y}F(y/\mu)\,dy.
\end{equation}
Condition~\eqref{eq:general-margin-modulus} is the only local requirement on $F$. It does not impose a power-law exponent or a pointwise positive gap. By dominated convergence, $\Phi_F(\mu)\to0$ as $\mu\to\infty$.

\begin{theorem}[Meta-LinEXP3 transfer regret]\label{theo:general-margin}
Suppose Assumptions~\ref{ass0}--\ref{ass1} and Assumption~\ref{ass:general-margin} hold. Before each task $s\ge2$, suppose there are a deterministic prior-error radius
$\varepsilon_s>0$, failure level $\delta_s\in[0,1]$, and prior concentration parameter $\mu_s>0$ such that
\begin{equation}\label{eq:prior-prediction-error}
 \mathbb P\!\left(
 \|\boldsymbol h_s-\boldsymbol v_s\|>\varepsilon_s
 \right)\le \delta_s,
\end{equation}
and set $\mu_1=0$. Define
\begin{equation}\label{eq:general-margin-B}
 B_1=\log k,
 \qquad
 B_s=(k-1)e^{\mu_s\Delta\varepsilon_s}\Phi_F(\mu_s)
 +\delta_s(\log k+\mu_s\Delta),\qquad s\ge2.
\end{equation}
Then $\Gamma_s(\pi_s)\le B_s$ for every task.

If PC-KDE is used and the task-dependent choices in
\eqref{eq:pc-tuning} are feasible with $\bar\Gamma_s=B_s$, then
\begin{equation}\label{eq:general-margin-pc}
 \mathrm{Reg}_{1:m}^{\mathrm{PC}}
 \le2\sqrt{nA_\zeta}\sum_{s=1}^m\sqrt{B_s}.
\end{equation}
If PRME is used and the conditions of Theorem~\ref{theo:prme-core} hold taskwise with
$\bar\Gamma_s=B_s$, then
\begin{align}
 \mathrm{Reg}_{1:m}^{\mathrm{PR}}
 \le{}&
 3\{\psi(1)\chi dY^2G\}^{1/3}n^{2/3}
 \sum_{s=1}^m B_s^{1/3}\notag\\
 &+4R\sqrt{\frac{\chi}{\lambda}}
 \sum_{u=1}^{mn}\xi_{k(u-1)}
 +\sum_{s=1}^m f_s.
 \label{eq:general-margin-prme}
\end{align}

Moreover, if
\begin{equation}\label{eq:qualitative-calibration}
 \mu_s\to\infty,\qquad
 \limsup_{s\to\infty}\mu_s\varepsilon_s<\infty,
 \qquad
 \delta_s(1+\mu_s)\to0,
\end{equation}
then $B_s\to0$. Consequently,
\begin{equation}\label{eq:cesaro-transfer}
 \sum_{s=1}^m\sqrt{B_s}=o(m),
 \qquad
 \sum_{s=1}^m B_s^{1/3}=o(m).
\end{equation}
\end{theorem}

\begin{corollary}[Eventual feasibility and aggregate rate]
\label{cor:general-margin-feasibility}
Under the conditions of Theorem~\ref{theo:general-margin}, including
\eqref{eq:qualitative-calibration}, the tuned PC-KDE and PRME choices satisfy their
feasibility conditions for all sufficiently large $s$, and any finite infeasible prefix
can be bounded by the trivial $Gn$ term. The PC-KDE and PRME learning terms scale as
$\sqrt n\,o(m)$ and $n^{2/3}o(m)$, respectively, the PRME moment term is
$\widetilde{\mathcal O}(\sqrt{mn/k})$, and the aggregate PRME failure remainder is
$\mathcal O(\delta)$ for fixed problem constants. Thus, for fixed $n$, $k$, and the
remaining problem constants, both bounds are sublinear in the number of tasks.
\end{corollary}

Condition~\eqref{eq:qualitative-calibration} is qualitative and requires only a weak rate of prior improvement. For instance, with $\mu_s=c\log(s+1)$ it is enough to have
$\varepsilon_s=\mathcal O(1/\log s)$ and
$\delta_s=o(1/\log s)$ for $B_s\to0$ under any admissible $F$ satisfying
\eqref{eq:general-margin-modulus}. No polynomial margin or polynomial tracking rate is required.

Theorem~\ref{theo:general-margin} is conditional on the prior-accuracy assumption \eqref{eq:prior-prediction-error}. It does not prove that the PCRW or Uniform priors used in the experiments satisfy this condition. Appendix~\ref{app:proofs} gives a power-law specialization, and Appendix~\ref{app:recurrent-realization} gives one sufficient construction based on an exogenous descriptor, PC-KDE, and finitely many recurrent task types. This construction is distinct from the experimental PCRW/Uniform protocol and does not cover PRME or LPE.

\begin{table}[H]
\centering
\small
\caption{Core guarantees. Task-number improvements require predictable structure. Unrelated tasks can still have linear total regret.}\label{tab1}
\begin{tabular}{>{\raggedright\arraybackslash}p{0.19\textwidth}>{\raggedright\arraybackslash}p{0.43\textwidth}>{\raggedright\arraybackslash}p{0.28\textwidth}}
\toprule
Setting & Core regret guarantee & Main qualification\\
\midrule
PC-KDE, one task &
$2\sqrt{nA_\zeta\bar\Gamma_s}$ &
Known $\mathcal D$ and exact policy moments\\
PRME, one task &
$\mathcal O(n^{2/3}\bar\Gamma_s^{1/3})$ plus explicit moment error &
Unknown $\mathcal D$ with condition-number dependence\\
LPE, one task &
Eq.~\eqref{eq:lpe-regret} with an explicit projection-bias residual &
No distributional input, $\mathcal O(d)$ update, and a residual that can be $\Theta(n)$ without directional coverage\\
Margin condition &
PC: \eqref{eq:general-margin-pc} and PRME: \eqref{eq:general-margin-prme} &
No power-law gap, and $B_s\to0$ is sufficient\\
\bottomrule
\end{tabular}
\end{table}

\paragraph{Machine-checked verification.}
All theoretical results in this paper have been formalized and
kernel-checked in Lean~4.31.0 with mathlib~v4.31.0. The formalization
contains no \texttt{sorry}, \texttt{admit}, custom axioms, or unsafe proof
bypasses. 

\section{Experiments}\label{sec:experiments}
We evaluate Meta-LinEXP3 on four settings: bounded synthetic transfer, a controlled PC-KDE/PRME/LPE comparison, MovieLens recommendation, and structured hyperspectral tensor sampling on KSC. The synthetic study uses 100 independent runs per similarity setting, the estimator comparison uses 60 runs, MovieLens uses 30 evaluation runs, and the KSC bandit curves use 60 runs. Unless noted otherwise, each curve shows the sample mean with one empirical standard deviation. Competing methods within a run share the same task/context realization and, where applicable, the same underlying random number stream.

LinEXP3 is recovered by setting the task prior to zero. PCRW uses the fixed positive-cosine retrieval weighting in \eqref{eqw1} with $\tau=1$, while Uniform uses \eqref{eqw2}. The synthetic experiment also includes an infeasible oracle that forms a normalized signed-cosine combination from the true task means. We use this oracle only to indicate the amount of transfer available from ground-truth task directions. It is not an upper bound on all possible transfer procedures.

\subsection{Synthetic Transfer Experiment}\label{sec6_1}
We consider $m=20$ sequential tasks with $n=30$ rounds per task, $k=40$ actions, and context dimension $d=5$. For every Monte Carlo run, the context law is regenerated as a bounded elliptical distribution. Specifically, with $\boldsymbol R\in\mathbb R^{d\times d}$ having i.i.d.\ standard Gaussian entries, we set
\[
 \boldsymbol\Sigma=0.05\boldsymbol R^T\boldsymbol R+0.01\boldsymbol I_d,
 \qquad \|\bar{\boldsymbol b}\|=0.1,
 \qquad \boldsymbol b=\bar{\boldsymbol b}+\sqrt d\,\operatorname{chol}(\boldsymbol\Sigma)\boldsymbol u,
\]
where $\boldsymbol u$ is sampled uniformly from the unit sphere. Hence contexts are bounded and have the constructed mean and covariance. The generated distribution is known and its raw second moment is positive definite, so the PRME clipping threshold can be computed exactly.

The round-wise loss vectors satisfy $\|\boldsymbol\theta_{s,t}\|\le1$. Each task mean $\boldsymbol\Theta_s=n^{-1}\sum_t\boldsymbol\theta_{s,t}$ has norm at most $0.5$, while the zero-mean within-task perturbations have radius at most $0.5$, so the generated $\boldsymbol\Theta_s$ is exactly the realized task mean. Task similarity is controlled by a lower bound $\mathrm{CS}_{\min}$ on the pairwise cosine similarity of nonzero task means, with values $-1,-0.5,0.5,$ and $1$. The case $\mathrm{CS}_{\min}=-1$ removes the pairwise lower-bound constraint, but the generator still uses a common random reference direction. It should therefore not be interpreted as a worst-case stream of unrelated tasks.

All four algorithms use PRME. We set
$\eta=\sqrt{\log(k)/n}=0.3507$, $\gamma=\sqrt{\log(k)/(mn)}=0.0784$, and $\mu=\log(k)/\log(n)=1.0846$, with PRME confidence parameter $\delta=0.05$. These fixed values are used for controlled empirical comparison and are not claimed to be theoretically optimal problem-dependent tuning.

\begin{figure}[t]
    \centering
    \begin{minipage}{0.45\linewidth}
        \centering
        \includegraphics[width=\linewidth]{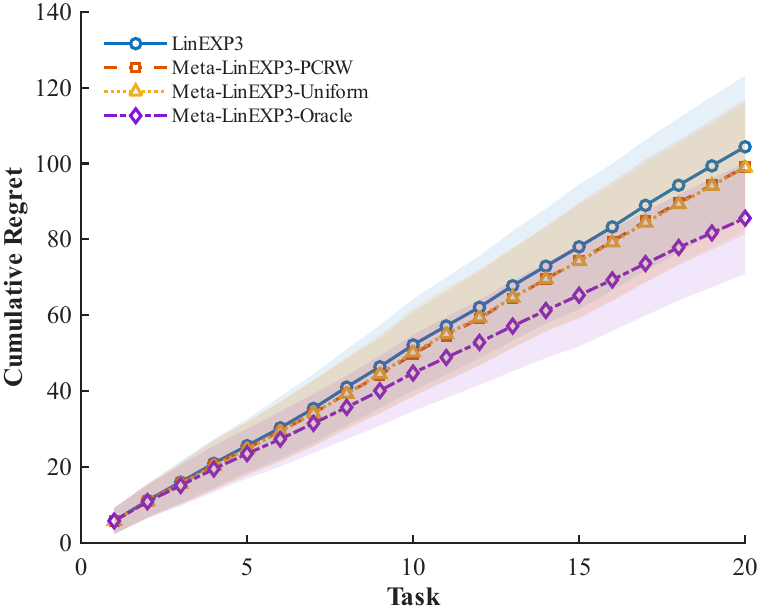}
        \subcaption{$\mathrm{CS}_{\min}=-1$}
    \end{minipage}%
    \hspace{0.04\linewidth}
    \begin{minipage}{0.45\linewidth}
        \centering
        \includegraphics[width=\linewidth]{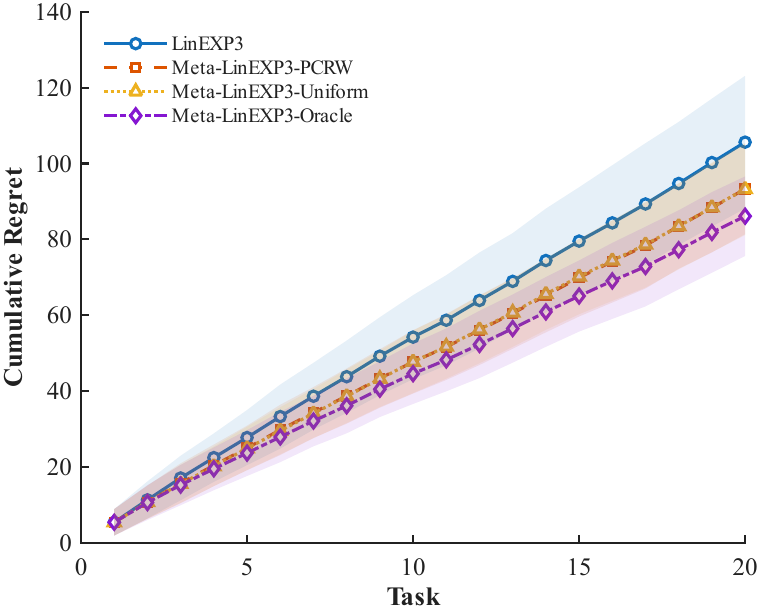}
        \subcaption{$\mathrm{CS}_{\min}=-0.5$}
    \end{minipage}

    \vskip\baselineskip

    \begin{minipage}{0.45\linewidth}
        \centering
        \includegraphics[width=\linewidth]{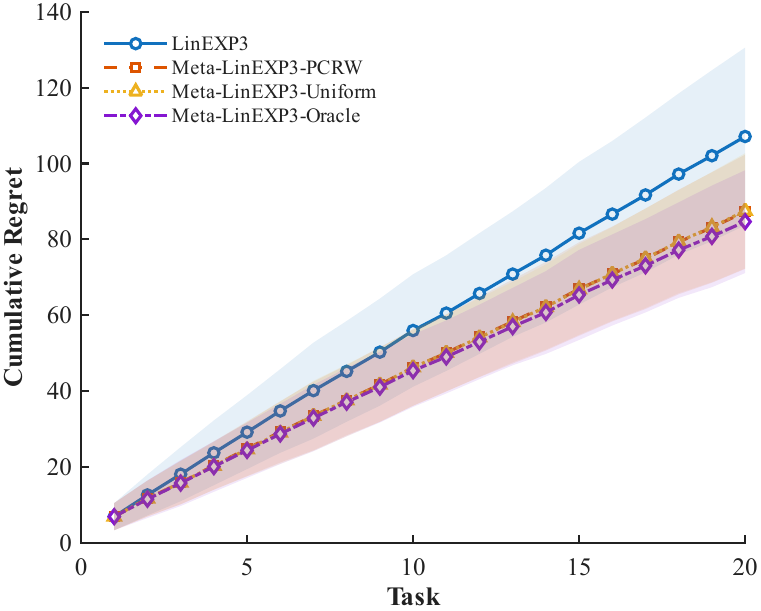}
        \subcaption{$\mathrm{CS}_{\min}=0.5$}
    \end{minipage}%
    \hspace{0.04\linewidth}
    \begin{minipage}{0.45\linewidth}
        \centering
        \includegraphics[width=\linewidth]{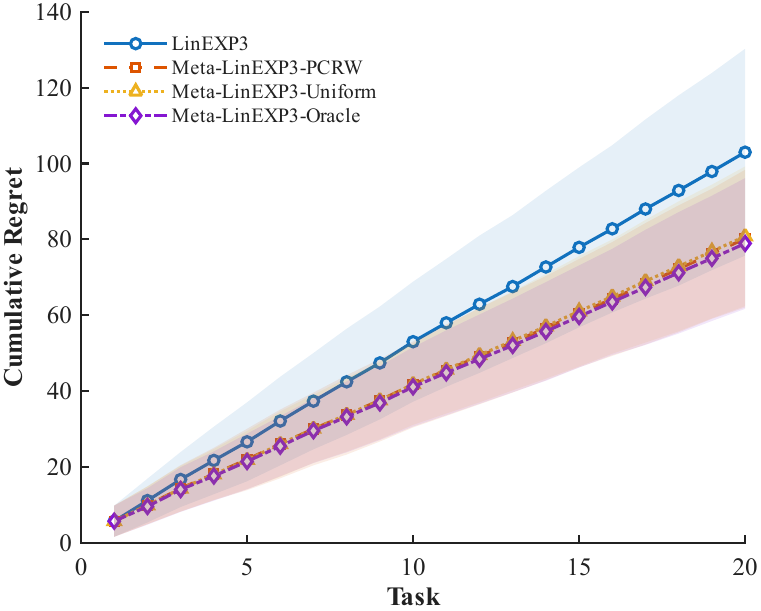}
        \subcaption{$\mathrm{CS}_{\min}=1$}
    \end{minipage}

    \caption{Cumulative regret on the bounded synthetic tasks using PRME. Each panel averages 100 runs. Shaded regions show one standard deviation.}
    \label{fig1}
\end{figure}

Figure~\ref{fig1} shows that the transfer gain increases as the imposed task alignment becomes stronger. At the final task, the mean cumulative regret of LinEXP3 is $104.40$, $105.63$, $107.12$, and $102.98$ for $\mathrm{CS}_{\min}=-1,-0.5,0.5,1$, respectively. The corresponding PCRW values are $99.10$, $93.29$, $87.27$, and $80.20$, which are reductions of approximately $5.1\%$, $11.7\%$, $18.5\%$, and $22.1\%$. Uniform is nearly indistinguishable from PCRW in this generator, with final means $99.02$, $93.24$, $87.54$, and $80.83$. The constrained task construction gives all tasks a common directional component, so a simple average is already informative. The oracle remains best in every panel, although its advantage over the practical priors narrows as alignment increases. The trend is consistent with the fixed prior becoming more informative as task directions align. This experiment does not address adaptive adversarial task sequences.

\subsection{Loss-Estimator Comparison}\label{sec:estimator-exp}
We next compare PC-KDE, PRME, and LPE while holding the task generator and meta-prior mechanism fixed. The experiment uses $m=12$ tasks, $n=40$ rounds, $k=3$ actions, $d=3$, task similarity setting $\mathrm{CS}_{\min}=0.5$, and 60 independent runs. Contexts are drawn uniformly from the four vertices of a regular tetrahedron in $\mathbb R^3$, scaled so that each support point has unit norm. This distribution is centered and has raw second moment $\boldsymbol I_3/3$. Because $k=3$, all $4^3=64$ ordered context tuples can be enumerated, and the policy-selected mean and covariance required by PC-KDE are computed exactly for the current adaptive policy rather than approximated by Monte Carlo integration.

All three methods use PCRW and share the same generated tasks, contexts, action-randomization stream, and hyperparameters $\eta=0.05$, $\gamma=0.45$, and $\mu=\log(3)/\log(40)=0.2978$. These values satisfy the PC-KDE score-stability condition used in this experiment. We report two quantities. For task $s$, the cumulative vector estimation error through round $t$ is
\[
 E_{s,t}=\left\|\sum_{j=1}^t\widehat{\boldsymbol\theta}_{s,j}-\sum_{j=1}^t\boldsymbol\theta_{s,j}\right\|,
\]
and the plotted value averages $E_{s,t}$ over the 12 tasks and then over runs. The second quantity is cumulative regret across tasks.

\begin{figure}[t]
    \centering
    \begin{minipage}{0.45\linewidth}
        \centering
        \includegraphics[width=\linewidth]{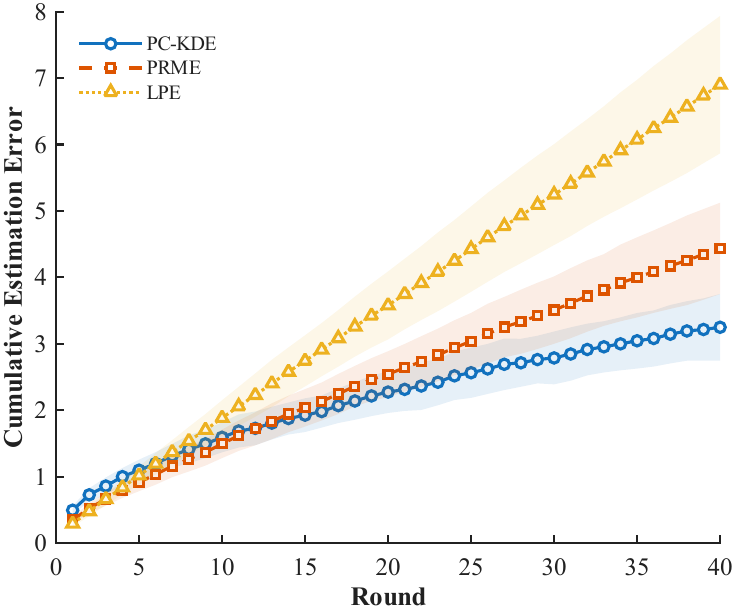}
        \subcaption{Cumulative estimation error within a task.}
    \end{minipage}%
    \hspace{0.04\linewidth}
    \begin{minipage}{0.45\linewidth}
        \centering
        \includegraphics[width=\linewidth]{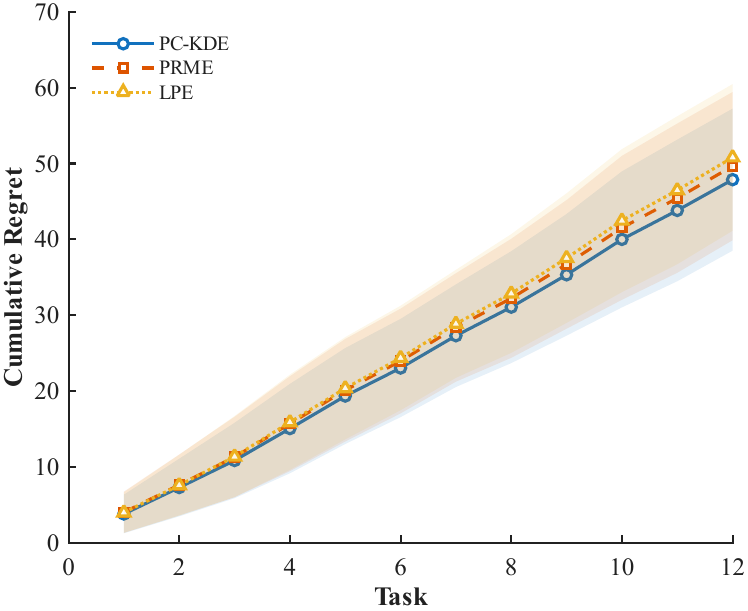}
        \subcaption{Cumulative regret across tasks.}
    \end{minipage}
    \caption{Controlled comparison of PC-KDE, PRME, and LPE on the exact finite-support context distribution. Each curve averages 60 runs and the shaded region shows one standard deviation.}
    \label{fig:estimator-comparison}
\end{figure}

At round 40, the mean cumulative estimation errors are $3.25\pm0.50$ for PC-KDE, $4.44\pm0.69$ for PRME, and $6.90\pm1.03$ for LPE. The ordering is consistent with the estimator constructions: PC-KDE uses exact policy-induced moments and is unbiased on the decision-relevant subspace, PRME incurs finite-sample moment estimation and inverse-propensity error, and LPE recovers only the component along the selected context. The final cumulative regrets are much closer---$47.87\pm9.36$, $49.66\pm9.77$, and $50.80\pm9.66$---showing that full-vector estimation error and decision regret need not move proportionally. When exact selected moments are available, PC-KDE gives the lowest estimation error in this controlled setting. PRME removes the need to know the context distribution at a modest empirical cost here, while LPE trades vector accuracy for an $\mathcal O(d)$ update.

\subsection{Movie Recommendation Task}\label{sec6_2}
We evaluate Meta-LinEXP3 on the MovieLens 100K data \cite{Data}, which contain 100,000 explicit ratings from 943 users on 1,682 movies. Movie genre indicators form $d=19$ context features, users are tasks, and movies are actions. Starting from the sparse user--movie rating matrix, we construct the experiment reward matrix using the attribute-clustering completion procedure detailed in Appendix~\ref{app:movielens-preprocess}: movies are grouped by their non-exclusive genre labels, users are assigned according to their observed genre preferences, and each missing entry is filled from the corresponding user-cluster/genre statistics whenever a valid statistic is available. Entries that remain zero are treated as unavailable and are excluded from candidate generation.

Eighty users are sampled once as a calibration pool and excluded from every reported evaluation run. The remaining users form the evaluation pool. Each of 30 runs samples $m=200$ evaluation users without replacement. Every user/task lasts $n=40$ rounds, and $k=10$ positive-rated candidate movies are sampled in each round. Meta-LinEXP3 and LinEXP3 use LPE, with $\eta=\sqrt{\log(10)/40}=0.2399$ and $\gamma=\sqrt{\log(10)/(200\cdot40)}=0.0170$. Because MovieLens ratings and binary genre vectors have a different scale from the synthetic experiment, the prior concentration parameter is calibrated only on the held-out 80 users. The tested grid is $\{0.6242,1.2484,2.4968\}$. Averaging three calibration repetitions selects $\mu=1.2484$, which is then frozen for all 30 evaluation runs.

We compare PCRW and Uniform against LinEXP3, Gaussian linear Thompson sampling (TS) \cite{TS,AgrawalGoyal2013}, and a conjugate hierarchical linear-Gaussian Meta-TS adaptation based on \cite{O_19}. Let $y_{s,t,a}^{\mathrm{ML}}\in[1,5]$ denote the positive completed rating used as reward. LinEXP3 and both Meta-LinEXP3 variants use the loss
\begin{equation}\label{eq:movielens-loss}
 \ell_{s,t,a}^{\mathrm{ML}}=-y_{s,t,a}^{\mathrm{ML}}\in[-5,-1],
\end{equation}
with no shift or normalization. TS and Meta-TS use $y_{s,t,a}^{\mathrm{ML}}$ directly. The reported metric is the actual cumulative best-candidate reward gap,
$\sum_{s=1}^m\sum_{t=1}^n(\max_{a\in[k]} y_{s,t,a}^{\mathrm{ML}}-y_{s,t,A_{s,t}}^{\mathrm{ML}})$, with no artificial affine shift. This empirical metric differs from the fixed-policy pseudo-regret analyzed in Section~\ref{sec:theory}. Moreover, the discrete rating data do not satisfy the noiseless linear-loss assumptions exactly.

\begin{figure}[t]
    \centering
    \includegraphics[width=0.52\linewidth]{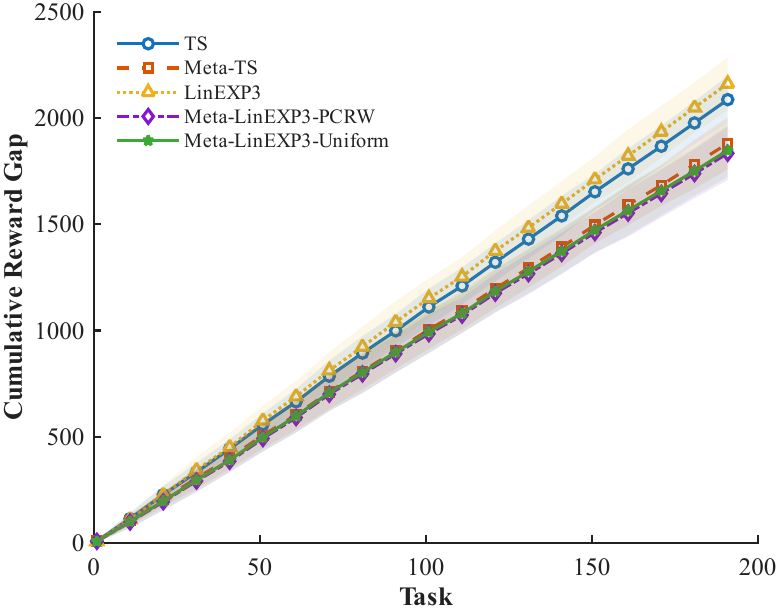}
    \caption{Cumulative best-candidate reward gap on MovieLens. Curves average 30 evaluation runs and shaded regions show one standard deviation. The plot reports the unshifted accumulated gap.}
    \label{figM}
\end{figure}

In Figure~\ref{figM}, the two Meta-LinEXP3 variants attain the lowest mean final reward gaps among the methods tested. At task 200, the mean cumulative gaps are $2256.7$ for LinEXP3, $2185.6$ for TS, $1966.6$ for Meta-TS, $1917.9$ for PCRW, and $1931.8$ for Uniform. Relative to LinEXP3, the mean gaps of PCRW and Uniform are lower by about $15.0\%$ and $14.4\%$, respectively. They are also about $2.5\%$ and $1.8\%$ below the mean gap of the tested Meta-TS adaptation. These differences are consistent with beneficial cross-user transfer, but they should not be interpreted as a general ranking of the methods: both the stochastic Gaussian model underlying TS/Meta-TS and the adversarial linear model underlying Meta-LinEXP3 are approximations to the MovieLens setting.

\subsection{Structured Tensor Sampling for Hyperspectral Data}\label{sec6_3}
We further evaluate Meta-LinEXP3 on structured sampling of the Kennedy Space Center (KSC) hyperspectral cube \cite{KSC}. The data tensor has size $512\times614\times176$, and every spectral slice is treated as a task, giving $m=176$ and context dimension $d=1126$. For each slice we compute rank-$K=10$ SVD factors. Each candidate sampling set contains 400 points, each task runs for $n=50$ rounds, and each round presents $k=20$ candidates drawn from a slice-specific FFW-derived proposal.

A complete evaluation candidate bank is generated once and reused in all 60 evaluation repetitions. Each candidate is sampled without replacement and must satisfy full-rank feasibility and a reciprocal-condition-number threshold of $10^{-3}$ for the two Gram matrices entering the reconstruction MSE. The within-slice learner uses the absolute-MSE LPE update with $\eta=2.45\times10^{-4}$ and $\gamma=0.01845$. The KSC loss is nonlinear in the incidence context, and its raw positive MSE contains a large taskwise offset. In a completed-task LPE summary, this absolute offset mainly reflects selection frequency rather than relative sampling-set quality. For cross-slice transfer, we therefore use the centered relative loss summary defined below. For a completed task define
\begin{equation}\label{eq:ksc-meta-summary}
 \boldsymbol z_s^{\rm meta}
 =\frac1n\sum_{t=1}^n
 \frac{(\ell_{s,t,A_{s,t}}-\bar\ell_s)\boldsymbol b_{s,t,A_{s,t}}}{L_{\mathrm{samp}}},
 \qquad
 \bar\ell_s=\frac1n\sum_{t=1}^n\ell_{s,t,A_{s,t}} .
\end{equation}
After subtracting the coordinate mean, we apply the same regularized direction normalization used elsewhere. A below-average selected set contributes negative cost mass to its coordinates, so overlap with historically good sampling coordinates receives more probability under the standard fixed prior $\exp(-\mu\langle\boldsymbol b,\boldsymbol h_s\rangle)$. This centered relative loss summary is application specific. It is not claimed to be an unbiased estimator of the linear loss vector from the theoretical ALCB model.

To separate scale selection from the reported suffix, slices $1$--$24$ form a chronological calibration prefix. The first 12 prefix slices initialize the retrieval history, and slices $13$--$24$ score a fixed early-search summary computed from the best-observed-MSE trajectory over rounds $1{:}10$. Calibration uses an independent candidate bank. The initial grid is
\[
 \mu\in\{0,\;0.0383,\;0.1,\;0.25,\;0.5,\;1,\;2,\;4,\;8\},
\]
where $\mu=0$ allows the procedure to reject transfer. If the PCRW optimum remains at the largest positive value, the grid is doubled until an interior optimum or saturation is observed. Twenty calibration repetitions give PCRW early-search scores $332.642$ at $\mu=8$, $331.763$ at $16$, and $331.810$ at $32$. Hence, the minimum occurs at $\mu=16$ and the $16$--$32$ comparison indicates saturation rather than a truncated boundary optimum. We therefore freeze $\mu=16$ before the held-out evaluation. The same frozen value is used for Uniform so that PCRW versus Uniform remains a controlled aggregation-rule ablation rather than a separately tuned baseline comparison.

{
Reported evaluation uses slices $25$--$176$ with a separate candidate
bank and 60 paired repetitions. PCRW, Uniform, and LinEXP3 share the same
candidate bank, 50-round search budget, $k$, LPE update, $\eta$, $\gamma$,
and paired inverse-CDF random numbers. LinEXP3 uses
$\boldsymbol h_s=\boldsymbol0_d$, whereas PCRW and Uniform use the
cross-slice priors constructed from completed slices. We also report
slice-wise FFW \cite{LIHAO} and reverse-greedy frame-potential sampling
(Greedy-FP) \cite{C}. Both construct one sampling set directly for each
slice. The bandit methods instead report the best set observed over
50 sequential MSE-feedback rounds, so comparisons with FFW and Greedy-FP
reflect recovery quality under different search procedures rather than
equal search effort.
}

\begin{figure}[t]
    \centering
    \begin{minipage}{0.45\linewidth}
        \centering
        \includegraphics[width=\linewidth]{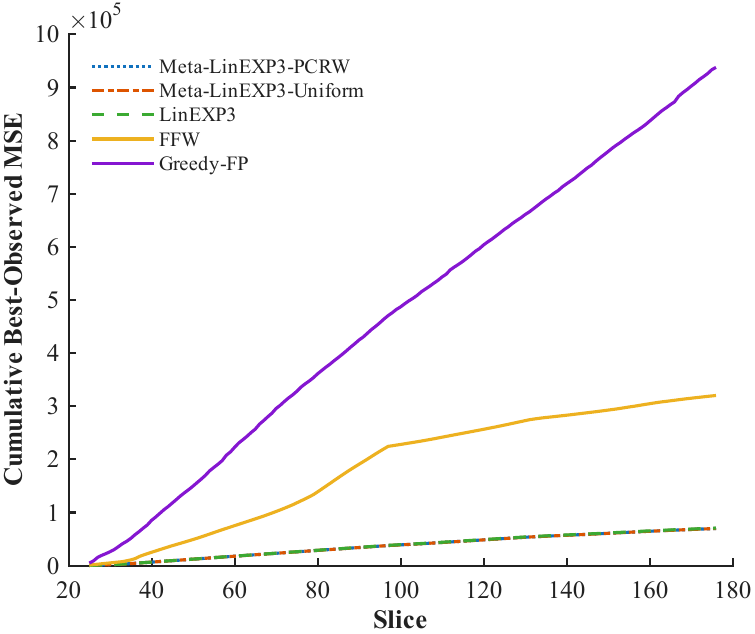}
        \subcaption{Cumulative best-observed MSE on slices 25--176.}
    \end{minipage}%
    \hspace{0.04\linewidth}
    \begin{minipage}{0.45\linewidth}
        \centering
        \includegraphics[width=\linewidth]{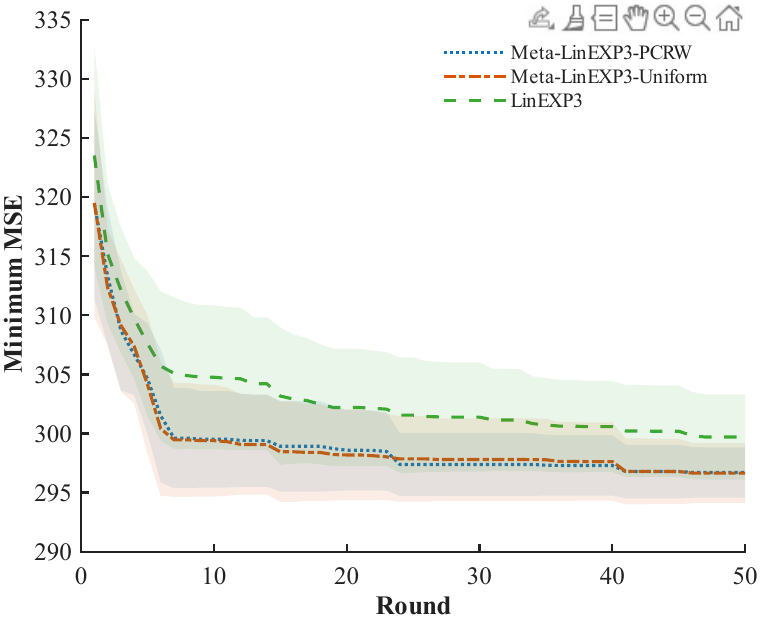}
        \subcaption{Best-observed MSE on spectral slice 176.}
    \end{minipage}
    \caption{KSC structured sampling. PCRW, Uniform, and LinEXP3 use the same 50 round sequential protocol. Their curves average 60 paired runs and shading denotes one standard deviation. FFW and Greedy-FP construct one sampling set directly for each slice.}
    \label{fig5}
\end{figure}

{
As shown in Figure~\ref{fig5}(a), cumulative best-observed MSE over the
reported evaluation slices is $69929.56\pm62.95$ for PCRW,
$69911.67\pm56.90$ for Uniform, and $70715.18\pm77.57$ for LinEXP3,
corresponding to reductions of $1.111\%$ and $1.136\%$ relative to
LinEXP3. The paired differences are $-785.63$ with 95\% CI
$[-807.45,-763.81]$ for PCRW--LinEXP3 and $-803.51$ with 95\% CI
$[-825.16,-781.86]$ for Uniform--LinEXP3. These results support the
benefit of the transferred cross-slice prior.

FFW and Greedy-FP finish at cumulative recovery MSEs $320255.19$ and
$937907.31$, respectively, substantially above those of the three bandit
methods. This advantage is consistent with the bandit methods directly
using recovery MSE as feedback while explicitly exploring alternative
candidate sets. In contrast, FFW and Greedy-FP optimize surrogate sampling
criteria that only indirectly target reconstruction error. Because FFW
and Greedy-FP construct one sampling set directly whereas the bandit
methods search sequentially over 50 MSE-feedback rounds, this comparison
should be interpreted as a recovery-quality comparison under different
search procedures rather than an equal-search-effort comparison.
}

{
The common linear scale in Figure~\ref{fig5}(a) visually compresses the
difference among the three bandit methods because FFW and Greedy-FP extend
the axis to much larger MSE values. The approximately $1.1\%$ reduction
relative to LinEXP3 is therefore more clearly quantified by the paired
differences and confidence intervals above.
}

{
Figure~\ref{fig5}(b) reports the best MSE observed by each bandit method
within the first $t$ rounds. Because additional rounds give every method
more chances to encounter a good candidate, the best-so-far gaps naturally
narrow as $t$ increases. Thus, the more informative comparison is how
quickly a method discovers a good sampling set when only a few search
rounds are available.

The fixed first-10-round best-observed-MSE summary is
$473.509\pm0.434$ for PCRW, $473.568\pm0.485$ for Uniform, and
$479.719\pm0.533$ for LinEXP3, corresponding to improvements of
$1.294\%$ and $1.282\%$. On slice 176, PCRW reaches mean best-observed MSE
$304.65$ by round 5, $299.51$ by round 10, and $298.92$ by round 15,
compared with $307.62$, $304.77$, and $303.17$ for LinEXP3. By round 50,
the means are $296.70$ and $299.70$. Both Meta-LinEXP3 variants outperform
LinEXP3 in the early rounds, and the trajectory flattens markedly after
roughly seven rounds. Thus, the transfer benefit is most pronounced when
only a few search rounds are available. These results show that the
cross-slice prior improves early search efficiency, allowing Meta-LinEXP3
to identify high-quality sampling sets with fewer within-slice observations.
This should be interpreted as a warm-start advantage rather than evidence
of faster asymptotic convergence.

The transfer diagnostics are consistent with this interpretation:
mean prior-score/MSE correlations are positive ($0.332$ for PCRW and
$0.494$ for Uniform), mean total-variation distances from zero-prior
LinEXP3 are $0.652$ and $0.438$, and paired action-disagreement rates are
$87.1\%$ and $78.5\%$. These diagnostics confirm that the transferred
summary materially changes the action distribution, consistent with the
observed early-round improvement.
}

{
We also measure algorithmic wall-clock time over the 152 evaluation slices.
The complete 50-round searches require $0.7244\pm0.0441$\,s for PCRW,
$0.6940\pm0.0440$\,s for Uniform, and $0.3996\pm0.0263$\,s for
LinEXP3, while FFW and Greedy-FP require $0.0750$\,s and $1.5236$\,s,
respectively. These timings exclude common SVD/data preprocessing,
candidate-bank generation, candidate-MSE precomputation, plotting,
post-hoc diagnostics, and downstream reconstruction evaluation.
Transfer adds moderate overhead over LinEXP3, while the measured
50-round Meta-LinEXP3 search remains below half the Greedy-FP search
time. Since the best-observed-MSE curves flatten markedly after roughly
seven rounds, applications with strict latency or computation constraints
can reduce the within-slice search budget to about seven rounds,
substantially shortening the sequential-search runtime while incurring
only a modest increase in best-observed MSE. This is a practical
implication of the observed trajectory rather than a separately
benchmarked seven-round runtime result.
}

To assess sensitivity to the final candidate bank, we independently regenerate six additional banks and run 12 paired action-random repetitions on each, keeping the prefix-calibrated $\mu=16$ fixed. Treating each bank mean as the bootstrap cluster, PCRW--LinEXP3 cumulative MSE is $-772.77$ with 95\% bootstrap CI $[-800.38,-744.04]$, a $1.093\%$ improvement. Uniform--LinEXP3 is $-783.48$ with CI $[-810.20,-752.60]$, a $1.108\%$ improvement. All six banks favor both transfer variants. The corresponding first-10-round early-search improvements are $1.239\%$ for PCRW and $1.211\%$ for Uniform, again with all six banks favorable. The close agreement between the fixed-bank and independent-bank results indicates that the reported gain is not specific to a single favorable candidate realization.

{
The KSC experiment is an application-level study rather than a direct
test of the linear theory. The inverse-Gram MSE is nonlinear in the
incidence context and the candidate distribution is slice dependent, so
the experiment does not validate the linear/common-context-distribution
theory. Instead, it evaluates whether the transferred prior remains useful
in a structured nonlinear sampling application. The results show that
cross-slice information improves candidate discovery when only a few
search rounds are available, while the advantage naturally narrows as
additional search opportunities become available.
}

\section{Conclusion}
Meta-LinEXP3 constructs a predictable prior from completed task summaries and keeps that prior fixed throughout the next ALCB task. For known context distributions, PC-KDE gives an intrinsic-dimension $\mathcal O(\sqrt n)$ per-task guarantee. For unknown distributions, PRME gives an explicit $\mathcal O(n^{2/3})$ leading term together with a finite-sample moment estimation error. Under the stated margin and prior-accuracy conditions, these taskwise guarantees translate into sublinear transfer-dependent terms over structured task sequences.

The experiments support four observations. In the bounded synthetic generator, transfer gains increase with task alignment. In the exact finite-support study, PC-KDE gives the lowest cumulative vector estimation error. On MovieLens, both Meta-LinEXP3 variants have lower mean final reward gaps than the tested baselines. {
On KSC, the shared sequential comparison gives about a $1.1\%$ cumulative
reduction relative to LinEXP3. Because best-observed MSE is a best-so-far
quantity, the smaller late-round separation is expected as all methods
receive more opportunities to find good candidates. The main
meta-learning benefit is therefore the warm-start advantage when only a
few search rounds are available. The same direction is observed on six
independently regenerated candidate banks. Measured runtime places both
Meta-LinEXP3 variants between LinEXP3 and Greedy-FP, showing that the
transfer gain is obtained with moderate additional cost.
}

The main limitations are computational and structural. Exact PC-KDE requires policy-induced moments, while PRME incurs inverse-propensity variance and $\mathcal O(d^2)$ memory. Improvement with the number of tasks is conditional on predictable structure that makes the prior increasingly accurate. Unrelated tasks need not yield any transfer gain. In addition, MovieLens and KSC do not exactly satisfy the noiseless-linear/common-context-distribution model, and the KSC centered relative loss summary is an application-specific outer representation rather than a theoretically unbiased loss vector estimator. Extending the analysis to approximate policy moments and to nonlinear application-specific task representations remains an important direction.
\section*{Funding}
The work was supported by the Major Scientific and Technological Innovation Platform Project of Hunan Province (2024JC1003).

\section*{Conflict of Interest}
The authors declare no relevant financial or nonfinancial conflicts of interest.



\bibliographystyle{plainnat}
\bibliography{references}

\clearpage
\appendix
\section{Proofs of the Theoretical Results}\label{app:proofs}

For readability, boldface is omitted locally in the derivations below when vector or matrix type is unambiguous. The symbols correspond to their boldface counterparts in the main text.

\subsection{Technical Results Deferred from the Main Text}\label{app:technical-results}

\paragraph{Changing priors and the general decomposition.}
For a predictable sequence of strictly positive kernels $r_{s,t}$, $t\in[n+1]$, define the scalar prior-drift charge
\begin{align}
 \omega_{s,t}(\mathcal B)
 &=\max_{a\in[k]}
 \log\frac{r_{s,t+1}(a\mid\mathcal B)}
 {r_{s,t}(a\mid\mathcal B)},\label{eq:prior-drift}\\
 \Gamma_s^{\mathrm{dyn}}(\pi)
 &=\mathbb E\!\left[
 -\log r_{s,n+1}(\pi(\mathcal B^0)\mid\mathcal B^0)
 +\sum_{t=1}^{n}\omega_{s,t}(\mathcal B^0)
 \right].\label{eq:transfer-complexity}
\end{align}
For softmax kernels
$r_{s,t}(a\mid\mathcal B)\propto
\exp[-\mu_s\langle\boldsymbol b_a,\boldsymbol h_{s,t}\rangle]$
with a common prior concentration parameter $\mu_s$,
\begin{equation}\label{eq:prior-drift-bound}
 \omega_{s,t}(\mathcal B)
 \le\mu_s\Delta
 \|\boldsymbol h_{s,t+1}-\boldsymbol h_{s,t}\|.
\end{equation}
Algorithm~\ref{alg1} has $\boldsymbol h_{s,t}=\boldsymbol h_s$, so
$\Gamma_s^{\mathrm{dyn}}(\pi)=\Gamma_s(\pi)$.

\begin{lemma}[Changing-prior exponential weights]\label{lem:changing-prior}
Fix $\mathcal B$ and scalar scores $z_{t,a}$, and write
$\boldsymbol z_t=(z_{t,1},\ldots,z_{t,k})^T$. Let
\begin{equation}\label{eq:q-from-scores}
 q_{s,t}(a\mid\mathcal B)=
 \frac{r_{s,t}(a\mid\mathcal B)
 \exp[-\eta\sum_{j<t}z_{j,a}]}
 {\sum_{a'=1}^k r_{s,t}(a'\mid\mathcal B)
 \exp[-\eta\sum_{j<t}z_{j,a'}]},
\end{equation}
and set
$\boldsymbol q_{s,t}(\mathcal B)=
(q_{s,t}(1\mid\mathcal B),\ldots,q_{s,t}(k\mid\mathcal B))^T$.
If $|\eta z_{t,a}|\le\zeta\le1$, then, for every $a^\star\in[k]$,
\begin{equation}\label{eq:changing-prior-bound}
 \sum_{t=1}^{n}\langle \boldsymbol q_{s,t}-\boldsymbol e_{a^\star},
 \boldsymbol z_t\rangle
 \le
 \frac{-\log r_{s,n+1}(a^\star\mid\mathcal B)
 +\sum_{t=1}^{n}\omega_{s,t}(\mathcal B)}{\eta}
 +\psi(\zeta)\eta\sum_{t=1}^{n}\sum_{a=1}^k
 q_{s,t}(a\mid\mathcal B)z_{t,a}^2.
\end{equation}
\end{lemma}

For a policy kernel $w$, set
$\boldsymbol m_w(\mathcal B)=\sum_aw(a\mid\mathcal B)\boldsymbol b_a$,
and let $\mathrm{unif}$ denote the uniform kernel. For an arbitrary estimator, define
\begin{equation}\label{eq:general-bias}
 \boldsymbol u_{s,t}
 =\mathbb E_{s,t}^{\mathrm{pre}}
 [\widehat{\boldsymbol\theta}_{s,t}]
 -\boldsymbol\theta_{s,t}.
\end{equation}

\begin{lemma}[General transfer and bias decomposition]\label{lem01}
Let $\boldsymbol a_{s,t}$ be $\mathcal H_{s,t}$-measurable and put
$z_{s,t,a}(\mathcal B^0)=
\langle\boldsymbol b_a^0-\boldsymbol a_{s,t},
\widehat{\boldsymbol\theta}_{s,t}\rangle$.
Suppose $|\eta z_{s,t,a}(\mathcal B^0)|\le\zeta\le1$ almost surely. Define
\begin{align}
 \mathcal W_{s,t}
 &=\mathbb E\!\left[
 \sum_{a=1}^kp_{s,t}(a\mid\mathcal B^0)
 z_{s,t,a}(\mathcal B^0)^2\right],
 \label{eq:general-second-moment}\\
 \mathcal E_{s,t}(\pi)
 &=\mathbb E\!\left[
 \left\langle\boldsymbol m_{\mathrm{unif}}(\mathcal B^0)
 -\boldsymbol b^0_{\pi(\mathcal B^0)},
 \boldsymbol\theta_{s,t}\right\rangle\right],
 \label{eq:exact-exploration}\\
 \mathfrak B_s(\pi)
 &=-(1-\gamma)\sum_{t=1}^{n}\mathbb E\!\left[
 \left\langle\boldsymbol m_{q_{s,t}}(\mathcal B^0)
 -\boldsymbol b^0_{\pi(\mathcal B^0)},
 \boldsymbol u_{s,t}\right\rangle\right].
 \label{eq:exact-bias}
\end{align}
Then
\begin{equation}\label{eq:general-regret}
 \mathrm{Reg}_s(\pi)
 \le\frac{(1-\gamma)\Gamma_s^{\mathrm{dyn}}(\pi)}{\eta}
 +\psi(\zeta)\eta\sum_{t=1}^{n}\mathcal W_{s,t}
 +\gamma\sum_{t=1}^{n}\mathcal E_{s,t}(\pi)
 +\mathfrak B_s(\pi),
\end{equation}
and $\mathcal E_{s,t}(\pi)\le G$.
\end{lemma}

\paragraph{Known-distribution refinements.}
\begin{proposition}[Data-dependent PC-KDE refinement]\label{prop:pc-refinement}
Under Theorem~\ref{theo01}, define
\begin{equation}\label{eq:pc-data-variance}
 \mathcal W^{\mathrm{PC}}_{s,t}
 =\mathbb E\!\left[
 \ell_{s,t,A_{s,t}}^2
 (\boldsymbol b_{s,t,A_{s,t}}-\boldsymbol x_{s,t})^T
 \boldsymbol H_{s,t}^{\dagger}
 (\boldsymbol b_{s,t,A_{s,t}}-\boldsymbol x_{s,t})
 \right].
\end{equation}
Then $\mathcal W^{\mathrm{PC}}_{s,t}\le Y^2d_{\mathcal U}$ and
\begin{equation}\label{eq:pc-regret-tight}
 \mathrm{Reg}_s(\pi)
 \le\frac{(1-\gamma)\Gamma_s(\pi)}{\eta}
 +\psi(\zeta)\eta\sum_{t=1}^{n}
 \mathcal W^{\mathrm{PC}}_{s,t}
 +\gamma\sum_{t=1}^{n}\mathcal E_{s,t}(\pi).
\end{equation}
\end{proposition}

\begin{theorem}[Average PC-KDE regret]\label{theo02}
Suppose the conditions of Theorem~\ref{theo01} hold for every task with common
$(\zeta,\eta,\gamma)$. If the comparators $\pi_s^\star$ exist, then
\begin{equation}\label{eq:average-known-regret}
 \frac{\mathrm{Reg}_{1:m}}m
 \le\frac{1-\gamma}{m\eta}\sum_{s=1}^m\Gamma_s(\pi_s^\star)
 +\psi(\zeta)\eta d_{\mathcal U}Y^2n+\gamma Gn.
\end{equation}
The same bound holds for any fixed comparator sequence. If a best comparator does not
exist, an $\varepsilon$-optimal comparator sequence incurs at most an additional
$\varepsilon$ in the average bound for any $\varepsilon>0$.
\end{theorem}

\begin{corollary}[Bounded-distribution transfer gain]\label{cor:bounded-transfer}
If $\boldsymbol\Theta_s$ and
$\pi_{\boldsymbol\Theta_s}\in\Pi$ are fixed before task $s$, then
\begin{equation}\label{eq:bounded-transfer-gain}
 \Gamma_s(\pi_{\boldsymbol\Theta_s})
 \le\mathbb E\!\left[
 \log k-\mu_s
 \mathbb E_{\mathcal B}\!\left[
 \langle\bar{\boldsymbol b}-
 \boldsymbol b_{\pi_{\boldsymbol\Theta_s}(\mathcal B)},
 \boldsymbol h_s\rangle\right]
 +\frac{\mu_s^2\Delta^2\|\boldsymbol h_s\|^2}{8}\right].
\end{equation}
\end{corollary}

\paragraph{PCRW tracking and the original fast-rate condition.}
For PCRW, set
\begin{align}
 D_s^{\mathrm{ret}}
 &=\sum_{i=1}^{s-1}c_{s,i}^{\mathrm{cos}}
 \|\widehat{\boldsymbol v}_i-\widehat{\boldsymbol v}_{s-1}\|,
 \notag\\
 e_s&=\|\widehat{\boldsymbol v}_s-\boldsymbol v_s\|,
 \qquad D_s^{\mathrm{task}}=\|\boldsymbol v_s-\boldsymbol v_{s-1}\|,\notag\\
 \delta_s^{\mathrm{track}}
 &=D_s^{\mathrm{ret}}+e_{s-1}+D_s^{\mathrm{task}}.
 \label{eq:tracking-complexity}
\end{align}

\begin{lemma}[Transfer Lipschitzness and PCRW tracking]\label{lem:pcrw-transfer}
For fixed $\mu_s$, let
$\Phi_s(\boldsymbol h)=
\mathbb E_{\mathcal B}[-\log
r_{\mu_s,\boldsymbol h}(\pi_s(\mathcal B)\mid\mathcal B)]$.
Then $\Phi_s$ is $\mu_s\Delta$-Lipschitz and
\begin{equation}\label{eq:pcrw-distance}
 \|\boldsymbol h_s^{\mathrm{cos}}-\boldsymbol v_s\|
 \le\delta_s^{\mathrm{track}},\qquad
 \Gamma_s(\pi_s)
 \le\Phi_s(\boldsymbol v_s)
 +\mu_s\Delta\delta_s^{\mathrm{track}}.
\end{equation}
If $\mu_1=0$, then
\begin{equation}\label{eq:pcrw-meta-excess}
 \sum_{s=1}^m\Gamma_s(\pi_s)
 \le\log k+\sum_{s=2}^m\Phi_s(\boldsymbol v_s)
 +\Delta\sum_{s=2}^m\mu_s\delta_s^{\mathrm{track}}.
\end{equation}
\end{lemma}

\begin{assumption}[Uniform gap and coarse tracking]\label{ass:meta-margin}
There is $g>0$ such that, almost surely for every $s\ge2$, every $\mathcal B$,
and every $a\ne\pi_s(\mathcal B)$,
\begin{equation}\label{eq:uniform-action-gap}
 \langle\boldsymbol b_a-
 \boldsymbol b_{\pi_s(\mathcal B)},\boldsymbol v_s\rangle\ge g,
\end{equation}
and $\delta_s^{\mathrm{track}}\le g/(2\Delta)$.
\end{assumption}

\begin{corollary}[Uniform-gap PC-KDE fast rate]\label{cor:pcrw-pc-m}
Under Assumption~\ref{ass:meta-margin}, set
\begin{equation}\label{eq:meta-temperature}
 \mu_1=0,\qquad
 \mu_s=\frac{4\log(s+1)}g\quad(s\ge2).
\end{equation}
Then
\begin{equation}\label{eq:gamma-decay}
 \Gamma_1(\pi_1)=\log k,\qquad
 \Gamma_s(\pi_s)\le\frac{k-1}{(s+1)^2}.
\end{equation}
Feasible taskwise PC-KDE tuning gives
\begin{equation}\label{eq:pc-cumulative-m}
 \mathrm{Reg}_{1:m}
 \le2\sqrt{nA_\zeta}\left[
 \sqrt{\log k}+\sqrt{k-1}\{1+\log(m+1)\}\right].
\end{equation}
\end{corollary}

Let
$Z_1,\ldots,Z_k\stackrel{\mathrm{i.i.d.}}{\sim}\mathcal N(0,1)$ and define
\[
 \nu_k=\mathbb E[\min_{a\le k}Z_a]<0,\qquad
 \cos_{\boldsymbol C}(\boldsymbol h,\boldsymbol\Theta)
 =\frac{\boldsymbol h^T\boldsymbol C\boldsymbol\Theta}
 {\sqrt{\boldsymbol h^T\boldsymbol C\boldsymbol h}
  \sqrt{\boldsymbol\Theta^T\boldsymbol C\boldsymbol\Theta}}.
\]
Set this cosine to zero if either denominator factor vanishes.

\begin{corollary}[Gaussian transfer calculation]\label{cor:gaussian-transfer}
Let $k\ge2$ and
$\boldsymbol b_a\stackrel{\mathrm{i.i.d.}}{\sim}
\mathcal N(\bar{\boldsymbol b},\boldsymbol C)$.
For $r_{\boldsymbol h}(a\mid\mathcal B)\propto
\exp[-\mu\langle\boldsymbol b_a,\boldsymbol h\rangle]$,
\begin{align}
 \mathbb E[-\log r_{\boldsymbol h}
 (\pi_{\boldsymbol\Theta}(\mathcal B)\mid\mathcal B)]
 &\le\log k+\mu\nu_k
 \frac{\boldsymbol h^T\boldsymbol C\boldsymbol\Theta}
 {\sqrt{\boldsymbol\Theta^T\boldsymbol C\boldsymbol\Theta}}
 +\frac{\mu^2}{2}\boldsymbol h^T\boldsymbol C\boldsymbol h,
 \label{eq:gaussian-transfer}\\
 &\le\log k-\frac{\nu_k^2}{2}
 [\cos_{\boldsymbol C}
 (\boldsymbol h,\boldsymbol\Theta)_+]^2
 \quad\text{after optimizing $\mu\ge0$.}
 \label{eq:gaussian-cosine}
\end{align}
\end{corollary}

This Gaussian calculation is separate from the bounded-support theorem. Combining them would require a truncation argument.

\paragraph{Detailed PRME and LPE results.}
On the simultaneous moment event,
\begin{equation}\label{eq:moment-event}
 \|\overline{\boldsymbol S}_{s,t}-\boldsymbol S\|
 \le\xi_{N_{s,t}},\qquad
 \boldsymbol S\preceq\widehat{\boldsymbol S}_{s,t}
 \preceq\boldsymbol S+2\xi_{N_{s,t}}\boldsymbol I_d.
\end{equation}
For the current policy define
\begin{align}
 \boldsymbol M_{s,t}
 &=\mathbb E_{\mathcal B}\!\left[
 \sum_{a=1}^kp_{s,t}(a\mid\mathcal B)
 \boldsymbol b_a\boldsymbol b_a^T\right],\notag\\
 \rho_{s,t}^{\mathrm{eff}}
 &=\operatorname{tr}(\boldsymbol M_{s,t}\boldsymbol S^{-1})
 \le\chi d.
 \label{eq:effective-design}
\end{align}
For a comparator put
\[
 \boldsymbol d_{s,t}^{\pi}
 =\mathbb E_{\mathcal B^0}\!\left[
 \boldsymbol m_{q_{s,t}}(\mathcal B^0)
 -\boldsymbol b^0_{\pi(\mathcal B^0)}\right].
\]
The exact PRME parameter bias is
\begin{equation}
 \boldsymbol u^{\mathrm{PR}}_{s,t}
 =\left(\widetilde{\boldsymbol S}_{s,t}^{-1}
 \boldsymbol S-\boldsymbol I_d\right)\boldsymbol\theta_{s,t}.
 \label{eq:prme-exact-bias}
\end{equation}

\begin{theorem}[Finite-sample PRME decomposition]\label{theoB}
Under Assumptions~\ref{ass0}--\ref{ass1} and
\eqref{eq:raw-moment-lower}, if PRME is used and
$\eta Y\kappa\le\gamma$, then
\begin{align}
 \mathrm{Reg}_s(\pi)
 \le{}&
 \frac{(1-\gamma)\Gamma_s(\pi)}{\eta}
 +\frac{\psi(1)\eta Y^2}{\gamma}
 \sum_{t=1}^n\mathbb E[\rho_{s,t}^{\mathrm{eff}}]
 +\gamma Gn\notag\\
 &+4R\sqrt{\frac{\chi}{\lambda}}
 \sum_{t=1}^n\xi_{N_{s,t}}
 +f_s.
 \label{eq:prme-tight-regret}
\end{align}
Here $f_s$ is the failure remainder bounded in \eqref{eq:failure-remainder}.
\end{theorem}

Using $\rho_{s,t}^{\mathrm{eff}}\le\chi d$ in Theorem~\ref{theoB} gives the simpler bound
\begin{equation}\label{eq:prme-simple-regret}
 \mathrm{Reg}_s(\pi)
 \le\frac{\Gamma_s(\pi)}{\eta}
 +\frac{\psi(1)\eta\chi dY^2n}{\gamma}
 +\gamma Gn
 +4R\sqrt{\frac{\chi}{\lambda}}
 \sum_{t=1}^n\xi_{N_{s,t}}
 +f_s.
\end{equation}
On the moment event~\eqref{eq:moment-event}, the signed bias additionally satisfies
\begin{equation}\label{eq:prme-bias-certificate}
 |(\boldsymbol d_{s,t}^{\pi})^T
 \boldsymbol u^{\mathrm{PR}}_{s,t}|
 \le\|\boldsymbol d_{s,t}^{\pi}\|_{\boldsymbol S^{-1}}
 \sqrt{\boldsymbol\theta_{s,t}^T
 \boldsymbol E_{s,t}\widehat{\boldsymbol S}_{s,t}^{-1}
 \boldsymbol E_{s,t}\boldsymbol\theta_{s,t}},
 \quad
 \boldsymbol E_{s,t}=\widehat{\boldsymbol S}_{s,t}-\boldsymbol S.
\end{equation}

\begin{corollary}[Uniform-gap PRME fast rate]\label{cor:pcrw-prme-m}
Under Assumption~\ref{ass:meta-margin}, use
\eqref{eq:meta-temperature} and taskwise PRME tuning with
$\bar\Gamma_1=\log k$ and
$\bar\Gamma_s=(k-1)/(s+1)^2$. Assume the taskwise exploration and
stability conditions of Theorem~\ref{theo:prme-core} hold. Then
\begin{align}
 \mathrm{Reg}_{1:m}
 \le{}&
 3\{\psi(1)\chi dY^2G\}^{1/3}n^{2/3}
 \left[(\log k)^{1/3}
 +3(k-1)^{1/3}(m+1)^{1/3}\right]\notag\\
 &+4R\sqrt{\frac{\chi}{\lambda}}
 \sum_{u=1}^{mn}\xi_{k(u-1)}
 +\sum_{s=1}^m f_s,
 \label{eq:prme-cumulative-m}
\end{align}
which is
\begin{equation}\label{eq:prme-cumulative-order}
 \widetilde{\mathcal O}\!\left(
 n^{2/3}m^{1/3}+\sqrt{\frac{mn}{k}}\right)
\end{equation}
for fixed remaining parameters.
\end{corollary}

Set $Y_{\mathrm{LP}}=LR$. The LPE projection identity implies
\begin{equation}\label{eq:lpe-error}
 \|\widehat{\boldsymbol\theta}^{\mathrm{LP}}_{s,t}
 -\boldsymbol\theta_{s,t}\|^2
 =\|\boldsymbol\theta_{s,t}\|^2
 -\frac{\ell_{s,t,A_{s,t}}^2}
 {\|\boldsymbol b_{s,t,A_{s,t}}\|^2}.
\end{equation}
The displayed fraction is defined as zero when the selected context is zero.

\begin{theorem}[Regret and exact projection bias with LPE]\label{theoC}
If $\eta Y_{\mathrm{LP}}\le\zeta\le1$, then
\begin{equation}\label{eq:lpe-regret}
 \mathrm{Reg}_s(\pi)
 \le\frac{(1-\gamma)\Gamma_s(\pi)}{\eta}
 +\psi(\zeta)\eta Y_{\mathrm{LP}}^2n+\gamma Gn
 +\mathfrak B_s^{\mathrm{LP}}(\pi),
\end{equation}
where
\begin{equation}\label{eq:lpe-signed-bias}
 \mathfrak B_s^{\mathrm{LP}}(\pi)
 =-(1-\gamma)\sum_{t=1}^n\mathbb E\!\left[
 (\boldsymbol d_{s,t}^{\pi})^T
 \{\mathbb E_{s,t}^{\mathrm{pre}}
 [\boldsymbol P_{\boldsymbol b_{s,t,A_{s,t}}}]
 -\boldsymbol I_d\}\boldsymbol\theta_{s,t}\right].
\end{equation}
It satisfies
\begin{align}
 |\mathfrak B_s^{\mathrm{LP}}(\pi)|
 &\le(1-\gamma)\sum_{t=1}^n
 \mathbb E\!\left[
 \|\boldsymbol d_{s,t}^{\pi}\|
 \sqrt{\|\boldsymbol\theta_{s,t}\|^2
 -\frac{\ell_{s,t,A_{s,t}}^2}
 {\|\boldsymbol b_{s,t,A_{s,t}}\|^2}}\right],
 \label{eq:lpe-bias-tight}\\
 &\le\Delta(1-\gamma)
 \sqrt{n\,\mathbb E\!\left[
 \sum_{t=1}^n\left(
 \|\boldsymbol\theta_{s,t}\|^2
 -\frac{\ell_{s,t,A_{s,t}}^2}
 {\|\boldsymbol b_{s,t,A_{s,t}}\|^2}\right)\right]}.
 \label{eq:lpe-bias-simple}
\end{align}
\end{theorem}

\begin{corollary}[Average LPE bias]\label{cor3}
For fixed comparator policies $\pi_1,\ldots,\pi_m$,
\begin{equation}\label{eq:lpe-average-bias}
 \frac1m\sum_{s=1}^m
 |\mathfrak B_s^{\mathrm{LP}}(\pi_s)|
 \le\Delta(1-\gamma)
 \sqrt{n^2R^2-\frac{n}{m}\mathbb E\!\left[
 \sum_{s=1}^m\sum_{t=1}^n
 \frac{\ell_{s,t,A_{s,t}}^2}
 {\|\boldsymbol b_{s,t,A_{s,t}}\|^2}\right]}.
\end{equation}
\end{corollary}

Without directional coverage, the average LPE bias bound can remain of order $n$.

\subsection{Filtration, Ghost Sets, and Changing Priors}

Throughout the proofs we use the pre-context filtration in \eqref{eq:pre-filtration}. Every policy kernel below is measurable before the fresh context set is drawn. The ghost set $\mathcal B^0\sim\mathcal D^k$ is independent of the entire realized interaction and is introduced only for the analysis.

\begin{proof}[Proof of Lemma~\ref{lem:changing-prior}]
For fixed $\mathcal B$, abbreviate $r_t(a)=r_{s,t}(a\mid\mathcal B)$ and let
\begin{equation}
 V_{t-1,a}^{\mathrm{cum}}=\sum_{j<t}z_{j,a},
 \qquad
 W_t=\sum_{a=1}^k r_t(a)e^{-\eta V_{t-1,a}^{\mathrm{cum}}}.
\end{equation}
Write $q_{t,a}=q_{s,t}(a\mid\mathcal B)$. The policy in
Algorithm~\ref{alg1} is
$q_{t,a}=r_t(a)e^{-\eta V_{t-1,a}^{\mathrm{cum}}}/W_t$. Subtracting a round-dependent constant from every $z_{t,a}$ leaves this policy unchanged.

For $|x|\le\zeta$, the elementary inequality
\begin{equation}\label{eq:exp-psi}
 e^{-x}\le1-x+\psi(\zeta)x^2
\end{equation}
holds with $\psi(\zeta)$ from \eqref{eq:psi-c}. Hence
\begin{align}
 \widetilde W_{t+1}
 &:=\sum_a r_t(a)e^{-\eta(V_{t-1,a}^{\mathrm{cum}}+z_{t,a})}
 =W_t\sum_aq_{t,a}e^{-\eta z_{t,a}}\notag\\
 &\le W_t\exp\!\left(
 -\eta\langle q_t,z_t\rangle
 +\psi(\zeta)\eta^2\sum_aq_{t,a}z_{t,a}^2\right).
 \label{eq:within-prior-potential}
\end{align}
By the definition of $\omega_{s,t}$,
\begin{equation}
 W_{t+1}
 =\sum_ar_{t+1}(a)e^{-\eta V_{t,a}^{\mathrm{cum}}}
 \le e^{\omega_{s,t}(\mathcal B)}\widetilde W_{t+1}.
\end{equation}
Since $W_1=1$, iteration of the preceding two displays gives an upper bound on $\log W_{n+1}$. On the other hand, for any $a^\star$,
\begin{equation}
 W_{n+1}\ge r_{n+1}(a^\star)
 e^{-\eta\sum_{t=1}^nz_{t,a^\star}}.
\end{equation}
Combining the upper and lower bounds and dividing by $\eta$ proves \eqref{eq:changing-prior-bound}.
\end{proof}

\begin{proof}[Proof of \eqref{eq:prior-drift-bound}]
For $r_{\boldsymbol h}(a\mid\mathcal B)\propto\exp[-\mu_s\langle b_a,\boldsymbol h\rangle]$, let $Z_t(\mathcal B)=\sum_a\exp[-\mu_s\langle b_a,h_{s,t}\rangle]$. Then
\begin{equation}
 \log\frac{r_{s,t+1}(a\mid\mathcal B)}{r_{s,t}(a\mid\mathcal B)}
 =-\mu_s\langle b_a,h_{s,t+1}-h_{s,t}\rangle
 +\log Z_t-\log Z_{t+1}.
\end{equation}
The gradient is
\begin{equation}
 \nabla_{\boldsymbol h}\log r_{\boldsymbol h}(a\mid\mathcal B)
 =-\mu_s\left(b_a-\sum_c r_{\boldsymbol h}(c\mid\mathcal B)b_c\right),
\end{equation}
whose norm is at most $\mu_s\Delta$. The mean-value theorem and maximization over $a$ prove the claim.
\end{proof}

\begin{proof}[Proof of Lemma~\ref{lem01}]
Fix $s,t$ and condition on $\mathcal H_{s,t}$. For any policy kernel $w$, freshness of $\mathcal B_{s,t}$ gives
\begin{equation}\label{eq:fresh-real-ghost}
 \mathbb E\!\left[
 \left\langle m_w(\mathcal B_{s,t})-b_{\pi(\mathcal B_{s,t})},
 \theta_{s,t}\right\rangle\middle|\mathcal H_{s,t}\right]
 =\mathbb E_{\mathcal B^0}\!\left[
 \left\langle m_w(\mathcal B^0)-b^0_{\pi(\mathcal B^0)},
 \theta_{s,t}\right\rangle\right].
\end{equation}
Independence of the ghost set and the definition of $u_{s,t}$ also imply
\begin{align}
 &\mathbb E\!\left[
 \left\langle m_{q_{s,t}}(\mathcal B^0)-b^0_{\pi(\mathcal B^0)},
 \widehat\theta_{s,t}\right\rangle\middle|\mathcal H_{s,t}\right]\notag\\
 &\quad=\mathbb E_{\mathcal B^0}\!\left[
 \left\langle m_{q_{s,t}}(\mathcal B^0)-b^0_{\pi(\mathcal B^0)},
 \theta_{s,t}+u_{s,t}\right\rangle\right].
 \label{eq:biased-ghost-identity}
\end{align}
Decompose $p_{s,t}=(1-\gamma)q_{s,t}+\gamma\,\mathrm{unif}$. Equations \eqref{eq:fresh-real-ghost}--\eqref{eq:biased-ghost-identity}, summed over $t$, express the true regret as the exploitation surrogate, the exact uniform exploration term \eqref{eq:exact-exploration}, and the signed correction \eqref{eq:exact-bias}. Apply Lemma~\ref{lem:changing-prior} pathwise to $\mathcal B^0$ with $a^\star=\pi(\mathcal B^0)$. Finally,
\begin{equation}
 (1-\gamma)q_{s,t}(a\mid\mathcal B^0)
 \le p_{s,t}(a\mid\mathcal B^0),
\end{equation}
so the multiplied quadratic term is at most \eqref{eq:general-second-moment}. Taking the joint expectation proves \eqref{eq:general-regret}. The range bound on $\mathcal E_{s,t}$ follows directly from Assumption~\ref{ass0}.
\end{proof}

\subsection{Known-Distribution Estimator}

\begin{lemma}[Affine support and selected covariance]\label{lem:affine-support}
Let $C=\operatorname{Cov}_{\mathcal D}(b)$ and $\mathcal U=\operatorname{range}(C)$. Then $b-\bar b\in\mathcal U$ almost surely. Moreover, for the selected covariance $H_{s,t}$ in \eqref{eq:selected-covariance},
\begin{equation}\label{eq:selected-cov-lower}
 \operatorname{range}(H_{s,t})=\mathcal U,
 \qquad
 H_{s,t}|_{\mathcal U}\succeq\gamma C|_{\mathcal U}.
\end{equation}
\end{lemma}

\begin{proof}
Fix $s,t$ and condition on $\mathcal H_{s,t}$. All distributions and means in this proof are therefore conditional, pathwise quantities.
For every $v\in\mathcal U^\perp$,
$\mathbb E[\langle v,b-\bar b\rangle^2]=v^TCv=0$. Hence, the first assertion holds. Let $Q_{s,t}$ be the marginal distribution of the context selected by $q_{s,t}$. The marginal selected by $p_{s,t}$ is
$P_{s,t}=(1-\gamma)Q_{s,t}+\gamma\mathcal D$, and let $x^q_{s,t}=\mathbb E_{b\sim Q_{s,t}}[b]$. The covariance-mixture identity yields
\begin{align}
 H_{s,t}
 ={}&(1-\gamma)\operatorname{Cov}_{Q_{s,t}}(b)+\gamma C\notag\\
 &+\gamma(1-\gamma)
 (x^q_{s,t}-\bar b)(x^q_{s,t}-\bar b)^T
 \succeq\gamma C.
\end{align}
All centered selected contexts lie in $\mathcal U$, so the reverse range inclusion holds as well. This proves \eqref{eq:selected-cov-lower}.
\end{proof}

\begin{proof}[Proof of Theorem~\ref{theo01}]
Conditioning before the current set and using noiseless linear loss,
\begin{align}
 \mathbb E_{s,t}^{\mathrm{pre}}[\widehat\theta^{\mathrm{PC}}_{s,t}]
 &=H_{s,t}^{\dagger}
 \mathbb E[(b_{s,t,A_{s,t}}-x_{s,t})b_{s,t,A_{s,t}}^T]
 \theta_{s,t}\notag\\
 &=H_{s,t}^{\dagger}H_{s,t}\theta_{s,t}
 =P_{\mathcal U}\theta_{s,t}.
 \label{eq:pc-mean-proof}
\end{align}
Thus the full parameter bias is $-P_{\mathcal U^\perp}\theta_{s,t}$. Both
$m_{q_{s,t}}(\mathcal B^0)$ and $b^0_{\pi(\mathcal B^0)}$ belong to the same affine space $\bar b+\mathcal U$. Their difference belongs to $\mathcal U$. Its inner product with the bias is therefore zero, proving exact decision-relevant unbiasedness.

The selected mean belongs to the closed convex hull of the support, so
$\|b-x_{s,t}\|\le\Delta$. Lemma~\ref{lem:affine-support} gives
\begin{equation}
 |\langle b_a^0-x_{s,t},\widehat\theta^{\mathrm{PC}}_{s,t}\rangle|
 \le \frac{Y\Delta^2}{\gamma\lambda_+}.
\end{equation}
Hence \eqref{eq:pc-stability} implies the score condition in Lemma~\ref{lem01} with center $a_{s,t}=x_{s,t}$.

For the second moment, independence of the ghost set first gives
\begin{equation}
 \mathbb E_{\mathcal B^0}\sum_ap_{s,t}(a\mid\mathcal B^0)
 \langle b_a^0-x_{s,t},\widehat\theta_{s,t}^{\mathrm{PC}}\rangle^2
 =(\widehat\theta_{s,t}^{\mathrm{PC}})^TH_{s,t}
 \widehat\theta_{s,t}^{\mathrm{PC}}.
\end{equation}
Taking expectation over the actual set and action yields exactly \eqref{eq:pc-data-variance}. Since $|\ell|\le Y$,
\begin{align}
 \mathcal W^{\mathrm{PC}}_{s,t}
 &\le Y^2\mathbb E[(b_A-x_{s,t})^TH_{s,t}^{\dagger}(b_A-x_{s,t})]\notag\\
 &=Y^2\operatorname{tr}(H_{s,t}^{\dagger}H_{s,t})
 =Y^2d_{\mathcal U}.
\end{align}
Substitution in Lemma~\ref{lem01} proves both regret bounds.
\end{proof}

\begin{proof}[Proof of Corollary~\ref{cor1}]
Set $\gamma=\eta Y\kappa_{\mathcal U}/\zeta$ in \eqref{eq:pc-regret-simple}, use $1-\gamma\le1$ and $\Gamma_s(\pi)\le\bar\Gamma_s$, and obtain
\begin{equation}
 \mathrm{Reg}_s(\pi)\le\frac{\bar\Gamma_s}{\eta}+\eta nA_\zeta.
\end{equation}
The choice in \eqref{eq:pc-tuning} minimizes the right-hand side. Its feasibility gives the stated result.
\end{proof}

\begin{proof}[Proof of Theorem~\ref{theo02}]
Apply \eqref{eq:pc-regret-simple} to $\pi_s^\star$ for every task, sum, and divide by $m$. Since $\mathrm{Reg}_s=\mathrm{Reg}_s(\pi_s^\star)$, this gives \eqref{eq:average-known-regret}. If the supremum is not attained, let $\pi_{s,\varepsilon}$ be an $\varepsilon$-optimal comparator for task $s$. Then
\(
 \mathrm{Reg}_{1:m}/m
 \le \varepsilon+m^{-1}\sum_s\mathrm{Reg}_s(\pi_{s,\varepsilon})
\)
for every sequence specified in the theorem. Applying \eqref{eq:pc-regret-simple} to that
sequence gives the asserted $\varepsilon$-optimal bound. Applying the same argument
directly to an arbitrary fixed comparator sequence proves its stated bound. No
independence among task parameters is required for this step.
\end{proof}

\begin{proof}[Proof of Corollary~\ref{cor:bounded-transfer}]
Condition on the history preceding task $s$, and hence on $h=h_s$, and set $\Theta=\Theta_s$ and $\mu=\mu_s$. Since the ghost set is independent of this history, Jensen's inequality and the i.i.d.\ ghost arms give
\begin{align}
 \mathbb E[-\log r_h(\pi_\Theta(\mathcal B)\mid\mathcal B)]
 &\le\log k+\mu\mathbb E\langle b_{\pi_\Theta(\mathcal B)},h\rangle
 +\log\mathbb E e^{-\mu\langle b,h\rangle}.
\end{align}
The random variable $\langle b-\bar b,h\rangle$ has range length at most $\Delta\|h\|$. Hoeffding's lemma therefore bounds the last logarithm by
$-\mu\langle\bar b,h\rangle+\mu^2\Delta^2\|h\|^2/8$. Taking expectation over the pre-task history proves \eqref{eq:bounded-transfer-gain}. There is no drift term because the prior is fixed within the task.
\end{proof}

\begin{proof}[Proof of Lemma~\ref{lem:pcrw-transfer}]
For fixed $\mathcal B$ and comparator action $a^\star=\pi_s(\mathcal B)$,
\begin{equation}
 \nabla_h[-\log r_{\mu_s,h}(a^\star\mid\mathcal B)]
 =\mu_s\left(b_{a^\star}-\sum_a r_{\mu_s,h}(a\mid\mathcal B)b_a\right).
\end{equation}
Its norm is at most $\mu_s\Delta$, so averaging over $\mathcal B$ proves the Lipschitz claim.
Because the PCRW coefficients are nonnegative and sum to one,
\begin{align}
 \|h_s^{\mathrm{cos}}-\widehat v_{s-1}\|
 &=\left\|\sum_{i<s}c_{s,i}^{\mathrm{cos}}
 (\widehat v_i-\widehat v_{s-1})\right\|
 \le D_s^{\mathrm{ret}}.
\end{align}
The triangle inequality then gives
$\|h_s^{\mathrm{cos}}-v_s\|\le D_s^{\mathrm{ret}}+e_{s-1}+D_s^{\mathrm{task}}$.
Apply Lipschitzness and the fixed-prior identity \eqref{eq:fixed-transfer-complexity}, then sum over tasks. For $s=1$, $\mu_1=0$ gives the uniform action prior and complexity $\log k$.
\end{proof}

\begin{proof}[Proof of Theorem~\ref{theo:general-margin}]
Fix $s\ge2$, write $\varepsilon=\varepsilon_s$ and $\mu=\mu_s$, and let
$E_s=\{\|h_s-v_s\|\le\varepsilon\}$. Condition on $\mathcal G_s$ from \eqref{eq:augmented-pre-task}. Since $h_s$ is $\mathcal H_{s,1}$-measurable and $v_s$ is $\mathcal G_s$-measurable, $E_s$ is $\mathcal G_s$-measurable. On $E_s$, write $\mathcal B=\mathcal B^0$ for the fresh ghost set defined in the cross-task setup and put $a^\star=\pi_s(\mathcal B)$. For every $a\ne a^\star$,
\begin{align}
 \langle b_a-b_{a^\star},h_s\rangle
 &\ge \langle b_a-b_{a^\star},v_s\rangle
 -\|b_a-b_{a^\star}\|\,\|h_s-v_s\|\notag\\
 &\ge \operatorname{gap}_s(\mathcal B)-\Delta\varepsilon.
 \label{eq:general-margin-predicted-gap}
\end{align}
Therefore $\log(1+x)\le x$ gives
\begin{align}
 -\log r_s(a^\star\mid\mathcal B)
 &=\log\!\left[1+\sum_{a\ne a^\star}
 e^{-\mu\langle b_a-b_{a^\star},h_s\rangle}\right]\notag\\
 &\le(k-1)e^{\mu\Delta\varepsilon}
 e^{-\mu\operatorname{gap}_s(\mathcal B)}.
 \label{eq:general-margin-softmax}
\end{align}
Let $X=\operatorname{gap}_s(\mathcal B)$. Since
$\lim_{x\to 0^+}F(x)=0$, Assumption~\ref{ass:general-margin} implies
$\mathbb P(X=0\mid\mathcal G_s)=0$. Conditional integration by parts yields
\begin{align}
 \mathbb E[e^{-\mu X}\mid\mathcal G_s]
 &=\mu\int_0^\infty e^{-\mu x}
 \mathbb P(X\le x\mid\mathcal G_s)\,dx\notag\\
 &\le\mu\int_0^\infty e^{-\mu x}F(x)\,dx
 =\Phi_F(\mu).
 \label{eq:general-margin-laplace-proof}
\end{align}
Combining \eqref{eq:general-margin-softmax}--\eqref{eq:general-margin-laplace-proof}
shows that, on the $\mathcal G_s$-measurable event $E_s$, the conditional transfer complexity is at most
$(k-1)e^{\mu\Delta\varepsilon}\Phi_F(\mu)$. On $E_s^c$, the universal fallback
\eqref{eq:transfer-uniform-fallback} is $\log k+\mu\Delta$. Taking total expectation and using $\mathbb P(E_s^c)\le \delta_s$ proves
$\Gamma_s(\pi_s)\le B_s$. For $s=1$, $\mu_1=0$ gives the uniform prior and
$\Gamma_1(\pi_1)=B_1=\log k$.

Because each task is oblivious, minimizing its cumulative comparator loss is equivalent to
minimizing $\langle b_a,\Theta_s\rangle$ for each context set. Components in
$\mathcal U^\perp$ add the same offset to all actions, so the fixed policy $\pi_s$ in
\eqref{eq:true-task-direction} attains the supremum in \eqref{eq:policy-regret}. Hence
$\mathrm{Reg}_{1:m}=\sum_s\mathrm{Reg}_s(\pi_s)$. Applying
Corollary~\ref{cor1} task by task with $\bar\Gamma_s=B_s$ gives
\eqref{eq:general-margin-pc}. Applying Theorem~\ref{theo:prme-core} and summing its cube-root
terms gives \eqref{eq:general-margin-prme}. The moment terms combine because
$N_{s,t}=k\{n(s-1)+t-1\}$ in lexicographic order.

It remains to prove the qualitative statement. By the change of variables in
\eqref{eq:margin-laplace-envelope},
\[
 \Phi_F(\mu)=\int_0^\infty e^{-y}F(y/\mu)\,dy.
\]
For each fixed $y$, $F(y/\mu)\to0$ as $\mu\to\infty$, and the integrand is bounded by
$e^{-y}$. Dominated convergence therefore gives $\Phi_F(\mu)\to0$. Under
\eqref{eq:qualitative-calibration}, the exponential factor in the first term of $B_s$ is
uniformly bounded for all sufficiently large $s$, while the second term tends to zero.
Thus $B_s\to0$. The same holds for $\sqrt{B_s}$ and $B_s^{1/3}$, and Ces\`aro
summability gives \eqref{eq:cesaro-transfer}.
\end{proof}

\begin{proof}[Proof of Corollary~\ref{cor:general-margin-feasibility}]
Equations~\eqref{eq:general-margin-pc}, \eqref{eq:general-margin-prme}, and
\eqref{eq:cesaro-transfer} give the stated PC-KDE and PRME learning-term scales. The
tuned PC-KDE exploration probability is proportional to $\sqrt{B_s}$, while the tuned
PRME exploration probability is proportional to $B_s^{1/3}$ and the left-hand side of
\eqref{eq:prme-feasibility} is proportional to $B_s$. Hence both tuning conditions hold
for all sufficiently large $s$. A finite prefix may use the trivial $Gn$ bound.
Equation~\eqref{eq:sum-xi} gives the
$\widetilde{\mathcal O}(\sqrt{mn/k})$ PRME moment-learning term. Under the tuned PRME
parameters, the aggregate failure remainder is $\mathcal O(\delta)$ for fixed problem
constants, exactly as in the proof of Theorem~\ref{theo:prme-core}. These bounds give
the stated sublinear regret consequences.
\end{proof}

Let
$\Gamma_{\mathrm E}(u)=\int_0^\infty e^{-x}x^{u-1}\,dx$
denote the Euler gamma function.

\begin{corollary}[Power-law specialization]\label{cor:power-margin}
If, in addition, the function $F$ in Assumption~\ref{ass:general-margin} satisfies
\begin{equation}\label{eq:power-margin}
 F(x)\le C_gx^\alpha,\qquad C_g>0,\quad \alpha>0,
\end{equation}
then
\begin{equation}\label{eq:power-laplace}
 \Phi_F(\mu)\le C_g\Gamma_{\mathrm E}(\alpha+1)\mu^{-\alpha}.
\end{equation}
With $\mu_s=\alpha/(\Delta\varepsilon_s)$,
\begin{equation}\label{eq:power-margin-B}
 B_s\le (k-1)C_g\Gamma_{\mathrm E}(\alpha+1)e^\alpha
 \left(\frac{\Delta}{\alpha}\right)^\alpha\varepsilon_s^\alpha
 +\delta_s\left(\log k+\frac{\alpha}{\varepsilon_s}\right).
\end{equation}
\end{corollary}

Thus polynomial calibration rates for $\varepsilon_s$ and $\delta_s$ give corresponding polynomial task-number rates through Theorem~\ref{theo:general-margin}.

\begin{proof}[Proof of Corollary~\ref{cor:power-margin}]
Under \eqref{eq:power-margin},
\begin{align}
 \Phi_F(\mu)
 &\le C_g\mu\int_0^\infty e^{-\mu x}x^\alpha\,dx
 =C_g\Gamma_{\mathrm E}(\alpha+1)\mu^{-\alpha},
\end{align}
which proves \eqref{eq:power-laplace}. Substituting $\mu_s=\alpha/(\Delta\varepsilon_s)$ in
\eqref{eq:general-margin-B} gives \eqref{eq:power-margin-B}.
\end{proof}

\subsection{A Concrete Recurrent-Type Sufficient Condition for Prior Consistency}\label{app:recurrent-realization}

This subsection gives one structured sufficient condition for the prior-accuracy assumption in Theorem~\ref{theo:general-margin}. It is an analytical construction and is not an additional algorithmic variant used in Section~\ref{sec:experiments}. Before task $s$, suppose an exogenous descriptor $z_s\in[J]$ is revealed. For this subsection only, augment the pre-context filtration to
\begin{equation}\label{eq:descriptor-filtration}
 \mathcal H^{z}_{s,t}=\mathcal H_{s,t}\vee\sigma(z_s).
\end{equation}
The context-freshness condition of Section~\ref{sec2_1} is required conditional on $\mathcal H^{z}_{s,t}$. Because $z_s$ is revealed before the task prior and every within-task decision, the arguments leading to Theorem~\ref{theo:general-margin} remain valid under this augmented filtration.

\begin{assumption}[Finite recurrent types with pre-task descriptors]\label{ass:recurrent-prototype-app}
There are a fixed number $J<\infty$ of unit vectors
$\boldsymbol u_1,\ldots,\boldsymbol u_J\in\mathcal U$ and a descriptor sequence
$z_s\in[J]$ revealed before task $s$. The descriptor sequence and all task loss vectors are fixed before the interaction. For every task,
\begin{equation}\label{eq:recurrent-ray-app}
 \boldsymbol P_{\mathcal U}\boldsymbol\Theta_s
 =\rho_s\boldsymbol u_{z_s},
 \qquad \rho_s\ge\rho_{\min}>0.
\end{equation}
\end{assumption}

Define the same-descriptor history count and descriptor-gated prior by
\begin{equation}\label{eq:recurrent-prior-app}
 N_s=\sum_{i<s}\mathbf1\{z_i=z_s\},\qquad
 \bar{\boldsymbol g}_s=
 \begin{cases}
 \displaystyle\frac1{N_s}\sum_{i<s:\,z_i=z_s}
 \boldsymbol P_{\mathcal U}\widehat{\boldsymbol\Theta}_i,
 &N_s\ge1,\\[0.8em]
 \boldsymbol0_d,&N_s=0,
 \end{cases}
 \qquad
 \boldsymbol h_s^{\mathrm{rec}}=
 \begin{cases}
 \bar{\boldsymbol g}_s/\|\bar{\boldsymbol g}_s\|,
 &\bar{\boldsymbol g}_s\ne\boldsymbol0_d,\\
 \boldsymbol0_d,&\bar{\boldsymbol g}_s=\boldsymbol0_d.
 \end{cases}
\end{equation}
The raw projected summaries are averaged before normalization, so different positive magnitudes along the same decision direction reinforce the same recurrent type. The rule is $\mathcal H^z_{s,1}$-measurable and can be maintained with one running sum and count per descriptor.

On the $r$-th occurrence of any descriptor, impose the exploration floor
\begin{equation}\label{eq:recurrent-exploration-floor-app}
 \gamma_s\ge \underline\gamma_r
 :=\frac{\gamma_0}{\log(\mathrm e+r)},
 \qquad 0<\gamma_0\le\frac12.
\end{equation}
For $N\ge1$, set
\begin{align}
 \delta_N&=(\mathrm e+N)^{-2},\qquad
 \mu_N=c_\mu\log(\mathrm e+N),\quad c_\mu>0,\notag\\
 M_N&=R+\frac{Y\Delta}{\lambda_+\gamma_0}\log(\mathrm e+N),\notag\\
 d_N&=M_N\sqrt{\frac{2d_{\mathcal U}\log(2d_{\mathcal U}/\delta_N)}{N}},\notag\\
 \varepsilon_N&=2\min\left\{1,\frac{d_N}{\rho_{\min}}\right\},
 \label{eq:recurrent-radius-app}
\end{align}
and define
\begin{equation}\label{eq:recurrent-B-app}
 B_N^{\mathrm{rec}}
 =(k-1)e^{\mu_N\Delta\varepsilon_N}\Phi_F(\mu_N)
 +\delta_N(\log k+\mu_N\Delta).
\end{equation}

For a task with $N_s=N\ge1$, let
$(\eta_s^{\mathrm{tun}},\gamma_s^{\mathrm{tun}})$ be the PC-KDE tuning
\eqref{eq:pc-tuning} with $\bar\Gamma_s=B_N^{\mathrm{rec}}$. Whenever
$\gamma_s^{\mathrm{tun}}\le1/2$, use
\begin{equation}\label{eq:recurrent-tuning-app}
 \eta_s=\eta_s^{\mathrm{tun}},\qquad
 \gamma_s=\max\{\gamma_s^{\mathrm{tun}},\underline\gamma_{N+1}\}.
\end{equation}
First occurrences and the finite number of tuning-infeasible occurrences may use any admissible parameters satisfying the exploration floor. Let
$N_j(m)=\sum_{s=1}^m\mathbf1\{z_s=j\}$.

\begin{theorem}[Recurrent-type PC-KDE realization]\label{theo:recurrent-realization-app}
Suppose Assumptions~\ref{ass0}--\ref{ass1}, Assumption~\ref{ass:general-margin}, and Assumption~\ref{ass:recurrent-prototype-app} hold under the augmented filtration
\eqref{eq:descriptor-filtration}. Use PC-KDE and the descriptor-gated prior
\eqref{eq:recurrent-prior-app}, with the exploration floor
\eqref{eq:recurrent-exploration-floor-app}. Then, on every task with $N_s=N\ge1$,
\begin{equation}\label{eq:recurrent-confidence-app}
 \mathbb P\!\left(
 \|\boldsymbol h_s^{\mathrm{rec}}-\boldsymbol v_s\|>\varepsilon_N
 \right)\le\delta_N,
\end{equation}
and
\begin{equation}\label{eq:recurrent-vanish-app}
 \varepsilon_N=\mathcal O\!\left(\frac{(\log N)^{3/2}}{\sqrt N}\right),\qquad
 \mu_N\varepsilon_N\to0,\qquad
 \delta_N(1+\mu_N)\to0,\qquad
 B_N^{\mathrm{rec}}\to0.
\end{equation}
There is a constant $C_{\mathrm{init}}$ independent of $m$ such that
\begin{align}
 \mathrm{Reg}_{1:m}^{\mathrm{PC}}
 \le{}&C_{\mathrm{init}}
 +2\sqrt{nA_\zeta}
 \sum_{j=1}^J\sum_{N=1}^{N_j(m)-1}\sqrt{B_N^{\mathrm{rec}}}\notag\\
 &+Gn\gamma_0
 \sum_{j=1}^J\sum_{r=1}^{N_j(m)}
 \frac1{\log(\mathrm e+r)}.
 \label{eq:recurrent-regret-app}
\end{align}
For fixed $J,n$ and the remaining problem constants,
\begin{equation}\label{eq:recurrent-sublinear-app}
 \mathrm{Reg}_{1:m}^{\mathrm{PC}}=o(m).
\end{equation}
\end{theorem}

Thus historical PC-KDE summaries satisfy the prior-consistency premise of Theorem~\ref{theo:general-margin} in this finite recurrent-type model.

\begin{proof}
Fix a descriptor $j$ and enumerate its completed tasks in occurrence order as
$i_{j,1}<i_{j,2}<\cdots$. Write
\[
 \widehat\Theta_{j,r}=\widehat\Theta_{i_{j,r}},\qquad
 \rho_{j,r}=\rho_{i_{j,r}},\qquad
 D_{j,r}=\boldsymbol P_{\mathcal U}\widehat\Theta_{j,r}
 -\rho_{j,r}u_j.
\]
Because the descriptor sequence and task loss vectors are fixed before the interaction, the tower property and \eqref{eq:pc-parameter-bias} give
\begin{equation}\label{eq:recurrent-mds-app}
 \mathbb E[D_{j,r}\mid
 \mathcal F_{1:i_{j,r}-1},z_{i_{j,r}}]=0.
\end{equation}
Hence the descriptor-selected sequence is a vector martingale-difference sequence. The selected-covariance lower bound \eqref{eq:selected-cov-lower} and
$\|b_{s,t,A_{s,t}}-x_{s,t}\|\le\Delta$ imply, on the $r$-th occurrence,
\begin{equation}\label{eq:recurrent-summary-range-app}
 \|\widehat\theta^{\mathrm{PC}}_{i_{j,r},t}\|
 \le\frac{Y\Delta}{\gamma_{i_{j,r}}\lambda_+}
 \le\frac{Y\Delta}{\lambda_+\gamma_0}\log(\mathrm e+r).
\end{equation}
Since $\|\boldsymbol P_{\mathcal U}\Theta_{i_{j,r}}\|\le R$, averaging over the $n$ rounds yields
\begin{equation}\label{eq:recurrent-difference-range-app}
 \|D_{j,r}\|\le
 R+\frac{Y\Delta}{\lambda_+\gamma_0}\log(\mathrm e+r).
\end{equation}

Choose a fixed orthonormal basis of $\mathcal U$. For the first $N$ completed tasks of type $j$, apply scalar Azuma--Hoeffding to each of the $d_{\mathcal U}$ coordinates of \eqref{eq:recurrent-mds-app}, use \eqref{eq:recurrent-difference-range-app}, and take a union bound. With probability at least $1-\delta_N$,
\begin{align}
 \left\|\frac1N\sum_{r=1}^ND_{j,r}\right\|
 &\le\frac1N\sqrt{2d_{\mathcal U}\log(2d_{\mathcal U}/\delta_N)
 \sum_{r=1}^N M_r^2}\notag\\
 &\le M_N\sqrt{\frac{2d_{\mathcal U}\log(2d_{\mathcal U}/\delta_N)}{N}}
 =d_N.
 \label{eq:recurrent-azuma-app}
\end{align}
Before the $(N+1)$-st occurrence of descriptor $j$, the raw average in
\eqref{eq:recurrent-prior-app} therefore has the form
\[
 \bar g_s=\bar\rho_Nu_j+\bar D_N,
 \qquad
 \bar\rho_N=\frac1N\sum_{r=1}^N\rho_{j,r}\ge\rho_{\min},
 \qquad
 \bar D_N=\frac1N\sum_{r=1}^ND_{j,r},
 \qquad \|\bar D_N\|\le d_N.
\]
If $d_N<\rho_{\min}$, then $\bar g_s\ne0$ and
\begin{equation}\label{eq:recurrent-normalization-app}
 \left\|\frac{\bar g_s}{\|\bar g_s\|}-u_j\right\|
 \le\frac{2d_N}{\rho_{\min}}.
\end{equation}
If $d_N\ge\rho_{\min}$, the distance between two vectors in the unit ball is at most $2$. This proves \eqref{eq:recurrent-confidence-app}.

Since $M_N=\mathcal O(\log N)$ and
$\log(2d_{\mathcal U}/\delta_N)=\mathcal O(\log N)$,
\eqref{eq:recurrent-radius-app} gives
$\varepsilon_N=\mathcal O((\log N)^{3/2}/\sqrt N)$. Hence
$\mu_N\varepsilon_N=\mathcal O((\log N)^{5/2}/\sqrt N)\to0$ and
$\delta_N(1+\mu_N)\to0$. Theorem~\ref{theo:general-margin} then gives
$B_N^{\mathrm{rec}}\to0$.

On any task for which the tuned pair is feasible, apply \eqref{eq:pc-regret-simple} with
$\eta_s=\eta_s^{\mathrm{tun}}$ and the larger exploration value in
\eqref{eq:recurrent-tuning-app}. Since
$\gamma_s-\gamma_s^{\mathrm{tun}}\le\underline\gamma_{N+1}$,
\begin{equation}\label{eq:recurrent-per-task-app}
 \mathrm{Reg}_s(\pi_s)
 \le2\sqrt{nA_\zeta B_N^{\mathrm{rec}}}
 +Gn\underline\gamma_{N+1}.
\end{equation}
Because $B_N^{\mathrm{rec}}\to0$, the tuned exploration also tends to zero. Therefore, only finitely many occurrences of each of the fixed $J$ descriptors can violate
$\gamma_s^{\mathrm{tun}}\le1/2$. Together with the first occurrence of each descriptor, these tasks contribute a constant $C_{\mathrm{init}}$ under the trivial $Gn$ bound. Summing \eqref{eq:recurrent-per-task-app} in descriptor-occurrence order proves \eqref{eq:recurrent-regret-app}. Finally, Ces\`aro summability gives
$\sum_{N<M}\sqrt{B_N^{\mathrm{rec}}}=o(M)$, and
$\sum_{r\le M}1/\log(\mathrm e+r)=o(M)$. Summing over fixed $J$ proves \eqref{eq:recurrent-sublinear-app}.
\end{proof}

\begin{proof}[Proof of Corollary~\ref{cor:pcrw-pc-m}]
Let $a^\star=\pi_s(\mathcal B)$. For every $a\ne a^\star$, Assumption~\ref{ass:meta-margin} and \eqref{eq:pcrw-distance} give
\begin{align}
 \langle b_a-b_{a^\star},h_s\rangle
 &\ge \langle b_a-b_{a^\star},v_s\rangle
 -\|b_a-b_{a^\star}\|\|h_s-v_s\|\notag\\
 &\ge g-\Delta\frac{g}{2\Delta}=\frac g2.
\end{align}
Consequently,
\begin{align}
 -\log r_s(a^\star\mid\mathcal B)
 &=\log\!\left[1+\sum_{a\ne a^\star}
 e^{-\mu_s\langle b_a-b_{a^\star},h_s\rangle}\right]\notag\\
 &\le(k-1)e^{-\mu_sg/2}=\frac{k-1}{(s+1)^2},
\end{align}
where the last equality uses \eqref{eq:meta-temperature}. This proves \eqref{eq:gamma-decay}. Apply Corollary~\ref{cor1} task by task and use
$\sum_{s=2}^m(s+1)^{-1}\le1+\log(m+1)$ to obtain \eqref{eq:pc-cumulative-m}.
\end{proof}

\begin{proof}[Proof of Corollary~\ref{cor:gaussian-transfer}]
Let $g_a=\Theta^T(b_a-\bar b)$. Joint Gaussian regression gives
\begin{equation}
 \mathbb E[h^T(b_a-\bar b)\mid g_a]
 =\frac{h^TC\Theta}{\Theta^TC\Theta}g_a.
\end{equation}
The comparator selects the smallest $g_a$, whose expectation equals
$\nu_k\sqrt{\Theta^TC\Theta}$. Thus
\begin{equation}
 \mathbb E[h^T(b_{\pi_\Theta(\mathcal B)}-\bar b)]
 =\nu_k\frac{h^TC\Theta}{\sqrt{\Theta^TC\Theta}}.
\end{equation}
The Gaussian log moment-generating function is
$\log\mathbb E e^{-\mu h^T(b-\bar b)}=\mu^2h^TCh/2$. Repeating the Jensen step from the preceding proof gives \eqref{eq:gaussian-transfer}. Minimizing the resulting quadratic over $\mu\ge0$ proves \eqref{eq:gaussian-cosine}.
\end{proof}

\subsection{Unknown-Distribution Estimators}

\begin{lemma}[Past-only matrix concentration]\label{lem:matrix-bernstein}
With $\xi_N$ in \eqref{eq:bernstein-radius}, for every round with $N=N_{s,t}\ge1$,
\begin{equation}
 \mathbb P(\|\overline S_{s,t}-S\|>\xi_N)\le\frac{\delta}{T}.
\end{equation}
The corresponding event is $\mathcal H_{s,t}$-measurable.
\end{lemma}

\begin{proof}
Enumerate the $N$ previously revealed contexts as $b_1,\ldots,b_N$ and set
$X_i=b_ib_i^T-S$. Then $\mathbb E X_i=0$, $\|X_i\|\le L^2$, and
\begin{equation}
 \left\|\sum_{i=1}^N\mathbb E X_i^2\right\|
 \le NL^2\|S\|\le NL^4.
\end{equation}
The two-sided self-adjoint matrix Bernstein inequality \cite{Tropp2012}, inverted at
$\Lambda=\log(2dT/\delta)$, gives the stated radius. Only contexts preceding $(s,t)$ enter the event, proving measurability.
\end{proof}

\begin{proof}[Proof of Theorem~\ref{theoB}]
First, inverse propensity weighting gives the exact pre-context mean
\begin{align}
 \mathbb E_{s,t}^{\mathrm{pre}}[\widehat\theta^{\mathrm{PR}}_{s,t}]
 &=\frac1k\widetilde S_{s,t}^{-1}
 \mathbb E_{\mathcal B}\sum_{a=1}^kb_ab_a^T\theta_{s,t}
 =\widetilde S_{s,t}^{-1}S\theta_{s,t},
\end{align}
which proves \eqref{eq:prme-exact-bias}. Also, since $kp_{s,t}(a\mid\mathcal B)\ge\gamma$ and the safeguard ensures $\widetilde S_{s,t}\succeq\lambda I$,
\begin{equation}\label{eq:prme-global-range}
 |(b_a^0)^T\widehat\theta^{\mathrm{PR}}_{s,t}|
 \le\frac{Y L^2}{\gamma\lambda}=\frac{Y\kappa}{\gamma}.
\end{equation}
Thus the assumed parameter condition makes Lemma~\ref{lem01} applicable with $\zeta=1$ and center zero.

For the conditional second moment,
\begin{align}
 \mathbb E_{s,t}^{\mathrm{pre}}[\widehat\theta^{\mathrm{PR}}_{s,t}
 (\widehat\theta^{\mathrm{PR}}_{s,t})^T]
 &\preceq\frac{Y^2}{k^2}\widetilde S_{s,t}^{-1}
 \mathbb E_{\mathcal B}\sum_a\frac{b_ab_a^T}{p_{s,t}(a\mid\mathcal B)}
 \widetilde S_{s,t}^{-1}\notag\\
 &\preceq\frac{Y^2}{\gamma}
 \widetilde S_{s,t}^{-1}S\widetilde S_{s,t}^{-1}.
 \label{eq:prme-second-matrix}
\end{align}
On the event in Lemma~\ref{lem:matrix-bernstein},
$S\preceq\widehat S_{s,t}$, so clipping is inactive. Let
$A=S^{1/2}\widehat S_{s,t}^{-1}S^{1/2}\preceq I$. Then
\begin{equation}
 \widehat S_{s,t}^{-1}S\widehat S_{s,t}^{-1}
 =S^{-1/2}A^2S^{-1/2}\preceq S^{-1}.
\end{equation}
Consequently, the ghost second moment is at most
\begin{equation}
 \frac{Y^2}{\gamma}
 \operatorname{tr}(M_{s,t}S^{-1})
 =\frac{Y^2}{\gamma}\rho_{s,t}^{\mathrm{eff}}.
\end{equation}
The two Loewner bounds preceding \eqref{eq:effective-design} establish
$\rho_{s,t}^{\mathrm{eff}}\le\chi d$.

For the bias on the same event, put $E_{s,t}=\widehat S_{s,t}-S$. It is positive semidefinite and no larger than $2\xi_{N_{s,t}}I$. Since
$u^{\mathrm{PR}}_{s,t}=-\widehat S_{s,t}^{-1}E_{s,t}\theta_{s,t}$ and
$S\preceq\widehat S_{s,t}$,
\begin{equation}
 \|u^{\mathrm{PR}}_{s,t}\|_S^2
 \le\theta_{s,t}^TE_{s,t}\widehat S_{s,t}^{-1}
 E_{s,t}\theta_{s,t},
\end{equation}
which proves \eqref{eq:prme-bias-certificate}. Jensen's inequality gives
$\|m_w\|_{S^{-1}}\le\sqrt\chi$ for either the $q$-selected distribution or the comparator-selected distribution, because their raw second moments are bounded by both $kS$ and $L^2I$. Hence
\begin{equation}
 |(d_{s,t}^{\pi})^Tu^{\mathrm{PR}}_{s,t}|
 \le4\xi_{N_{s,t}}R\sqrt{\frac\chi\lambda}.
 \label{eq:prme-bias-uniform}
\end{equation}

It remains to account for failure without conditioning on an event involving future rounds. Each event in Lemma~\ref{lem:matrix-bernstein} is pre-context measurable. On its complement, \eqref{eq:prme-global-range} bounds the quadratic score by $(Y\kappa/\gamma)^2$. Moreover,
$\|u^{\mathrm{PR}}_{s,t}\|_S\le LR(\kappa+1)$ and
$\|d_{s,t}^{\pi}\|_{S^{-1}}\le2\sqrt\chi$. For a fixed task, the $n$ individual failure probabilities sum to at most $n\delta/T=\delta/m$, which gives \eqref{eq:failure-remainder}. Apply Lemma~\ref{lem01} with the exact coefficient $\psi(1)$, use $\mathcal E_{s,t}\le G$, and combine the good- and failure-event estimates. This proves \eqref{eq:prme-tight-regret}. Equation~\eqref{eq:effective-design} gives \eqref{eq:prme-simple-regret}.
\end{proof}

\begin{proof}[Proof of Theorem~\ref{theo:prme-core}]
Ignoring the displayed bias and failure terms temporarily, minimize
\begin{equation}
 \frac{\bar\Gamma_s}{\eta}
 +\frac{\psi(1)\eta\chi dY^2n}{\gamma}+\gamma Gn.
\end{equation}
The stationary point is exactly \eqref{eq:prme-tuning}. At this point, each of the three terms equals
\begin{equation}
 \{\psi(1)\chi d\bar\Gamma_sY^2G\}^{1/3}n^{2/3},
\end{equation}
which gives the leading term in \eqref{eq:prme-rate}. Condition \eqref{eq:prme-feasibility} is equivalent to $\eta_sY\kappa\le\gamma_s$ for the choices in \eqref{eq:prme-tuning}. The other stated condition controls the exploration mixture. Restoring the bias and failure terms proves \eqref{eq:prme-rate}.
\end{proof}

\begin{proof}[Proof of \eqref{eq:sum-xi}]
For $u\ge2$, use the convenient upper bound
\begin{equation}
 \xi_{k(u-1)}\le L^2\left[
 \sqrt{\frac{2\Lambda}{k(u-1)}}
 +\frac{2\Lambda}{3k(u-1)}\right].
\end{equation}
Now apply
$\sum_{j=1}^{T-1}j^{-1/2}\le2\sqrt{T-1}$ and
$\sum_{j=1}^{T-1}j^{-1}\le1+\log(T-1)$, and add $\xi_0=L^2$.
\end{proof}

\begin{proof}[Proof of Corollary~\ref{cor:pcrw-prme-m}]
Equation~\eqref{eq:gamma-decay} supplies the deterministic transfer bounds required by
Theorem~\ref{theo:prme-core}. Summing its leading term over tasks gives
\begin{align}
 3\{\psi(1)\chi dY^2G\}^{1/3}n^{2/3}
 \left[(\log k)^{1/3}+(k-1)^{1/3}
 \sum_{s=2}^m(s+1)^{-2/3}\right].
\end{align}
The integral bound
$\sum_{s=2}^m(s+1)^{-2/3}\le3(m+1)^{1/3}$
gives the first line of \eqref{eq:prme-cumulative-m}. In lexicographic task--round order,
$N_{s,t}=k(u-1)$ for $u=n(s-1)+t$, so the taskwise moment-bias sums combine exactly into
$\sum_{u=1}^{mn}\xi_{k(u-1)}$. The failure remainders add linearly. Finally apply
\eqref{eq:sum-xi}. Suppressing fixed problem parameters and logarithms gives
\eqref{eq:prme-cumulative-order}.
\end{proof}

\begin{proof}[Proof of Theorem~\ref{theoC} and Corollary~\ref{cor3}]
The pointwise identity
$\widehat\theta^{\mathrm{LP}}_{s,t}=P_{b_{s,t,A_{s,t}}}\theta_{s,t}$
implies \eqref{eq:lpe-error} by the Pythagorean theorem. It also gives
$\|\widehat\theta^{\mathrm{LP}}_{s,t}\|\le R$ and therefore
$|\langle b^0_a,\widehat\theta^{\mathrm{LP}}_{s,t}\rangle|\le LR=Y_{\mathrm{LP}}$.
The ghost second moment is at most $Y_{\mathrm{LP}}^2$. Lemma~\ref{lem01} then proves \eqref{eq:lpe-regret} and the exact signed expression \eqref{eq:lpe-signed-bias}.

Conditionally on $\mathcal H_{s,t}$, Jensen's inequality and \eqref{eq:lpe-error} give
\begin{align}
 \|u^{\mathrm{LP}}_{s,t}\|
 &\le\mathbb E_{s,t}^{\mathrm{pre}}
 \|(I-P_{b_{s,t,A_{s,t}}})\theta_{s,t}\|\notag\\
 &=\mathbb E_{s,t}^{\mathrm{pre}}
 \sqrt{\|\theta_{s,t}\|^2-
 \ell_{s,t,A_{s,t}}^2/\|b_{s,t,A_{s,t}}\|^2}.
\end{align}
Keeping $\|d_{s,t}^{\pi}\|$ inside the expectation yields \eqref{eq:lpe-bias-tight}. The diameter of the convex hull of the support is $\Delta$, so
$\|d_{s,t}^{\pi}\|\le\Delta$. Cauchy--Schwarz over $t$ proves \eqref{eq:lpe-bias-simple}. A second Cauchy--Schwarz/Jensen step over tasks, followed by $\|\theta_{s,t}\|\le R$, proves \eqref{eq:lpe-average-bias}. The residual is zero exactly when the loss vector lies in the selected one-dimensional span.
\end{proof}

\section{MovieLens Data Preparation and Attribute-Clustering Completion}\label{app:movielens-preprocess}
The MovieLens 100K data contain 943 users, 1,682 movies, 100,000 observed ratings on the integer scale 1--5, and 19 non-exclusive movie-genre indicators. We first construct the sparse rating matrix $\mathbf Y^{\mathrm{ML}}\in\mathbb R^{943\times1682}$, where $Y_{u,i}^{\mathrm{ML}}$ is the observed rating of user $u$ for movie $i$ and unobserved entries are initialized to zero. Let $\mathbf X^{\mathrm{genre}}\in\{0,1\}^{1682\times19}$ denote the genre-indicator matrix.

Movie clusters are defined by genre. Each of the 19 genres forms one cluster, and a movie belongs to every cluster corresponding to a nonzero entry in its genre vector. To obtain a user cluster, zeros in $\mathbf Y^{\mathrm{ML}}$ are temporarily treated as missing values. For every user and genre, we compute the mean of that user's observed ratings over movies carrying the genre. A user is assigned to the genre attaining the largest valid mean. If the user has no valid genre mean, the user is placed in a general cluster.

For every user cluster $c$ and genre $g$, let $\bar y_{c,g}^{\mathrm{ML}}$ be the mean of the available observed ratings contributed by users in cluster $c$ to movies in genre $g$. For an unobserved entry $Y_{u,i}^{\mathrm{ML}}=0$, let $c(u)$ be the user's cluster and let $\mathcal J(i)$ be the set of genres of movie $i$. The candidate imputation values are $\{\bar y_{c(u),g}^{\mathrm{ML}}:g\in\mathcal J(i)\}$ for which the cluster--genre mean exists. If at least one candidate is available, the missing entry is imputed by the maximum candidate. Otherwise, it remains zero and is treated as unavailable by the bandit experiment. 


\section{Supplementary Details for Section \ref{sec6_3}}\label{sec_STS}

\subsection{Background and Problem Statement}

Tensor slices along one mode often share structure. In a hyperspectral image, for example, spectral bands record the same scene at different wavelengths and therefore retain much of the same spatial organization. Figure~\ref{fig_ten1} shows several bands from the KSC data \cite{KSC}. We treat each band as a task and use the similarity between bands as the source of cross-task information.

\begin{figure}[ht]
    \centering
    \includegraphics[width=1\linewidth]{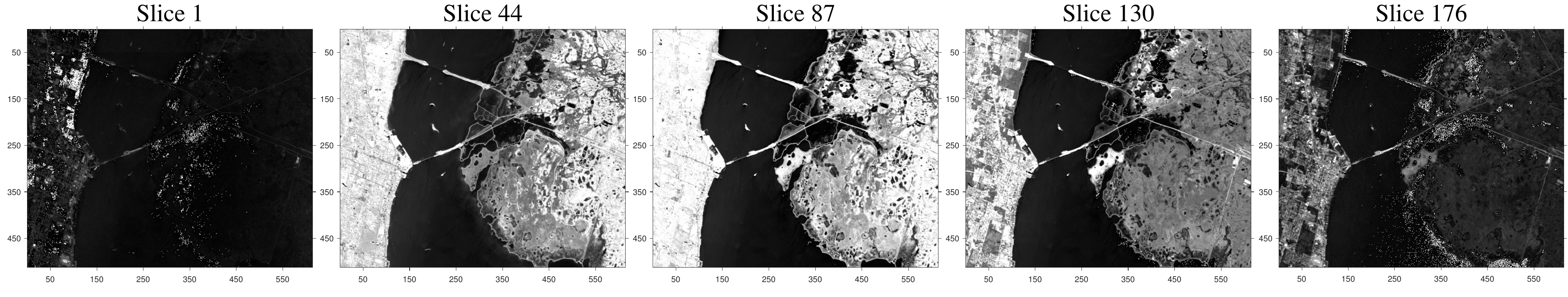}
    \caption{Slice images of the KSC hyperspectral dataset at selected spectral bands.}
    \label{fig_ten1}
\end{figure}

We consider a third-order tensor $\mathcal X \in \mathbb{R}^{N_1 \times N_2 \times N_3}$, where $\mathbf X^{(s)} = \mathcal X(:,:,s)$ is the $s$-th frontal slice. The slices $\{\mathbf X^{(s)}\}_{s=1}^{N_3}$ are assumed to share structural or statistical features along the third mode. We restrict the presentation to third-order tensors. A higher-order extension would require a corresponding definition of slices, feasible samples, and losses.

Set $N=N_1+N_2$ and use context dimension $d=N$. Let $\mathcal I=[N]$ index the available sensors, and let $\mathcal L_s\subset\mathcal I$ be the $L_{\mathrm{samp}}$ sensors selected for slice $s$. Its incidence vector $\mathbf1_{\mathcal L_s}\in\{0,1\}^N$ satisfies $\boldsymbol1_N^T\mathbf1_{\mathcal L_s}=L_{\mathrm{samp}}$. We first compute a rank-$K$ truncated singular value decomposition (SVD):
\[
\mathbf X^{(s)}\approx \mathbf{U}_{1}^{(s)}\mathbf{G}^{(s)}\left( \mathbf{U}_{2}^{(s)} \right) ^T,
\]
where $\mathbf U_1^{(s)}\in\mathbb R^{N_1\times K}$, $\mathbf U_2^{(s)}\in\mathbb R^{N_2\times K}$, and $\mathbf G^{(s)}\in\mathbb R^{K\times K}$.

Partition $\mathcal{L}_{s}$ into row indices for $\mathbf{U}_{1}^{\left( s \right)}$ and $\mathbf{U}_{2}^{\left( s \right)}$:
\[
\mathcal{L}_{s}' = \left\{ x \,\middle|\, x \in \mathcal{L}_{s},\, x \leq N_1 \right\},\quad \mathcal{L}_{s}'' = \left\{ x-N_1 \,\middle|\, x \in \mathcal{L}_{s},\, x > N_1 \right\}.
\]

Define
\[
\begin{aligned}
  &\mathbf{T}_{1}^{\left( s \right)}\left( \mathcal{L} _{s}^{\prime} \right) =\left( \mathbf{U}_{1}^{\left( s \right)}\left( \mathcal{L} _{s}^{\prime} \right) \right) ^T\mathbf{U}_{1}^{\left( s \right)}\left( \mathcal{L} _{s}^{\prime} \right),\\& \mathbf{T}_{2}^{\left( s \right)}\left( \mathcal{L} _{s}^{\prime\prime} \right) =\left( \mathbf{U}_{2}^{\left( s \right)}\left( \mathcal{L} _{s}^{\prime\prime} \right) \right) ^T\mathbf{U}_{2}^{\left( s \right)}\left( \mathcal{L} _{s}^{\prime\prime} \right),
\end{aligned}
\]
where $\mathbf{U}_{1}^{\left( s \right)}\left( \mathcal{L} _{s}^{\prime} \right)$ and $\mathbf{U}_{2}^{\left( s \right)}\left( \mathcal{L} _{s}^{\prime\prime} \right)$ denote the submatrices composed of the rows of $\mathbf{U}_{1}^{(s)}$ and $\mathbf{U}_{2}^{(s)}$ indexed by $\mathcal{L}_{s}'$ and $\mathcal{L}_{s}''$, respectively. The inverse-based MSE below is defined only when both submatrices have full column rank $K$. The cardinality conditions alone do not guarantee this property.

Under additive zero-mean white Gaussian noise with unit variance, the mean squared error (MSE) for slice $\mathbf X^{(s)}$ and sensor set $\mathcal{L}_{s}$ is
\[
\mathrm{MSE}\left( \mathbf X^{(s)},\mathcal{L} _s \right) =\operatorname{tr}\left( \left[ \mathbf{T}_{1}^{\left( s \right)}\left( \mathcal{L} _{s}^{\prime} \right) \right] ^{-1} \right) \cdot \operatorname{tr}\left( \left[ \mathbf{T}_{2}^{\left( s \right)}\left( \mathcal{L} _{s}^{\prime \prime} \right) \right] ^{-1} \right).
\]

The sampling problem for slice $s$ is therefore
\[
\underset{\mathcal L_s\subset\mathcal I}{\min}\ \mathrm{MSE}\left( \mathbf X^{(s)},\mathcal L_s \right) \quad \mathrm{s.t.}\  \left| \mathcal L_s \right|=L_{\mathrm{samp}},\quad \operatorname{rank}\!\left(\mathbf U_1^{(s)}(\mathcal L_s')\right)=\operatorname{rank}\!\left(\mathbf U_2^{(s)}(\mathcal L_s'')\right)=K.
\]
The rank constraints imply $|\mathcal L_s'|,|\mathcal L_s''|\ge K$ and ensure that the two Gram matrices in the MSE are invertible.

\subsection{Algorithm Implementation}
We treat the selected sensor set as the action and the inverse-Gram recovery MSE as bandit feedback. Because this MSE is nonlinear in the incidence vector, the case study uses the ALCB architecture outside the exact assumptions of the linear-loss theorem.

For every slice $s$, compute a rank-$K$ truncated SVD
\[
\mathbf X^{(s)}\approx \mathbf U_1^{(s)}\mathbf G^{(s)}(\mathbf U_2^{(s)})^T.
\]
The candidate distribution is slice specific. From the two factor matrices we compute the FFW derivative vector
\[
\boldsymbol d^{(s)}=
\left(d_1^{(1,s)},\ldots,d_{N_1}^{(1,s)},d_1^{(2,s)},\ldots,d_{N_2}^{(2,s)}\right)^T,
\]
where for $z\in\{1,2\}$ and row $y$,
\[
 d_y^{(z,s)}=
 \frac{N_z\,\boldsymbol p_y^{(z,s)T}
       [\mathbf U_z^{(s)T}\mathbf U_z^{(s)}]
       \boldsymbol p_y^{(z,s)}}
      {\|\mathbf U_z^{(s)T}\mathbf U_z^{(s)}\|_F^2},
\]
and $\boldsymbol p_y^{(z,s)}$ denotes row $y$ of $\mathbf U_z^{(s)}$ written as a column vector. We clip only negligible negative roundoff values and normalize $\boldsymbol d^{(s)}$ to a probability vector $\bar{\boldsymbol d}^{(s)}$.

For every pair $(s,t)$, the $k$ candidate sets $\mathcal L_{s,t,a}$, $a\in[k]$, are sampled without replacement from $\bar{\boldsymbol d}^{(s)}$. A candidate is accepted only if both selected factor submatrices have rank $K$ and the reciprocal condition numbers of the corresponding Gram matrices are at least $10^{-3}$. Candidate generation is repeated until the condition is met or a fixed maximum number of attempts is reached. The entire bank of $m n k$ candidate actions is generated once before the repeated bandit runs and is reused across all 60 repetitions. Hence repeated runs change only the action randomization, not the underlying candidate environment.

Given a selected candidate $\mathcal L_{s,t,A_{s,t}}$, the empirical loss is
\[
 \ell_{s,t,A_{s,t}}
 =\operatorname{tr}\!\left([\mathbf T_1^{(s)}(\mathcal L'_{s,t,A_{s,t}})]^{-1}\right)
  \operatorname{tr}\!\left([\mathbf T_2^{(s)}(\mathcal L''_{s,t,A_{s,t}})]^{-1}\right).
\]
Within each slice, Meta-LinEXP3 uses the absolute-MSE LPE update
\[
 \widehat{\boldsymbol\theta}_{s,t}
 =\frac{\boldsymbol b_{s,t,A_{s,t}}}
        {\|\boldsymbol b_{s,t,A_{s,t}}\|^2}
   \ell_{s,t,A_{s,t}},
\]
where $\boldsymbol b_{s,t,a}=\mathbf1_{\mathcal L_{s,t,a}}$ and $\|\boldsymbol b_{s,t,a}\|^2=L_{\mathrm{samp}}$. This update is used only by the within-task LinEXP3 scores. For the nonlinear KSC objective, the cross-task representation instead uses the centered relative loss summary in \eqref{eq:ksc-meta-summary}. Equivalently,
\[
 \boldsymbol z_s^{\rm meta}
 =\frac1n\left[
 \sum_{t=1}^n\frac{\ell_{s,t,A_{s,t}}\boldsymbol b_{s,t,A_{s,t}}}{L_{\mathrm{samp}}}
 -\bar\ell_s
 \sum_{t=1}^n\frac{\boldsymbol b_{s,t,A_{s,t}}}{L_{\mathrm{samp}}}
 \right].
\]
All candidates have the same cardinality, so subtracting the taskwise MSE offset preserves relative sampling-set information while removing the frequency-like component that can dominate an uncentered absolute-loss summary. Coordinate centering and normalization are performed only after a task is complete. Therefore, no current-task feedback enters that task's fixed prior.

The two comparison baselines are implemented independently for every slice. FFW uses the linearized derivative rule on row-normalized factors, with a rank-safe seed that guarantees $K$ linearly independent rows in each mode before the remaining budget is filled. Greedy-FP uses the published reverse-greedy product-frame-potential rule: it starts from all rows, repeatedly removes the row that gives the smallest product frame potential, and retains at least $K+\alpha$ rows in each mode. We use $\alpha=2$. Baseline MSE is evaluated with the original (not row-normalized) SVD factors.


\begin{algorithm}[t]
\caption{Meta-LinEXP3 for Structured Tensor Sampling}
\begin{algorithmic}[1]
\REQUIRE KSC tensor $\mathcal X$, rank $K$, budget $L_{\mathrm{samp}}$, $k$ candidates, $n$ rounds, $\eta,\gamma$, and the prefix-calibrated $\mu$.
\STATE Decode wrapped samples if present, apply one global positive rescaling, and compute rank-$K$ SVD factors for every slice.
\STATE On slices $1{:}24$, build an independent calibration candidate bank and select $\mu$ from the fixed grid using the first-10-round early-search score derived from the best-observed-MSE trajectory on slices $13{:}24$.
\STATE Build a separate final candidate bank for all slices and verify rank, conditioning, and candidate MSE.
\FOR{each paired evaluation run}
  \FOR{$s=1,\ldots,N_3$}
    \STATE Form the fixed PCRW or Uniform prior from completed centered relative loss summaries. Use $\boldsymbol h_s=\boldsymbol0$ for LinEXP3.
    \FOR{$t=1,\ldots,n$}
      \STATE Select one candidate with the common candidate bank/random variate, observe its recovery MSE, and update the within-task score with absolute-MSE LPE.
    \ENDFOR
    \STATE Store the best-observed MSE trajectory and the completed centered relative loss summary.
  \ENDFOR
\ENDFOR
\STATE Report slices $25{:}176$. Compare PCRW and Uniform with LinEXP3 under the same sequential protocol, and compute one-set FFW and Greedy-FP solutions separately.
\end{algorithmic}
\label{alg3}
\end{algorithm}

\end{document}